\documentclass[a4paper]{article}

\usepackage[pages=all, color=black, position={current page.south}, placement=bottom, scale=1, opacity=1, vshift=5mm]{background}

\usepackage[margin=1in]{geometry} 

\usepackage{dsfont}
\usepackage{bbm}
\usepackage{enumitem}
\usepackage{amsmath}
\usepackage{amsthm}
\usepackage{amssymb}

\usepackage[utf8]{inputenc}
\usepackage{hyperref}
\hypersetup{
	unicode,
	pdfauthor={Author One, Author Two, Author Three},
	pdftitle={A simple article template},
	pdfsubject={A simple article template},
	pdfkeywords={article, template, simple},
	pdfproducer={LaTeX},
	pdfcreator={pdflatex}
}

\usepackage[sort&compress,numbers,square]{natbib}

\theoremstyle{plain}
\newtheorem{theorem}{Theorem}

\newtheorem{lemma}[theorem]{Lemma}

\newtheorem{assumption}[theorem]{Assumption}

\theoremstyle{definition}

\newcommand{\save}[1]{{\vspace{0in}}#1{\vspace{0in}}}
\newcommand\given[1][]{\,\vert\,}
\newcommand{\eqd}{\mathop{=}\limits^d}
\newcommand{\alphahat}{\widehat\alpha}
\newcommand{\betahat}{\widehat\beta}
\newcommand{\Rcal}{\mathcal{R}}

\newcommand{\arrowas}{\overset{\rm a.s.}{\longrightarrow}}
\newcommand{\muhat}{\widehat\mu}
\newcommand{\Sigmahat}{\widehat\Sigma}
\setlist{leftmargin=10mm}
\usepackage{smile}

\usepackage{graphicx, color}
\graphicspath{{fig/}}

\usepackage{mathrsfs} 

\usepackage{lipsum}

\title{High Dimensional Binary Classification under Label Shift: Phase Transition and Regularization}

\author{Jiahui Cheng \and Minshuo Chen \and Hao Liu \and Tuo Zhao \and Wenjing Liao \thanks{Jiahui Cheng and Minshuo Chen contributed equally to this work. Jiahui Cheng and Wenjing Liao are affiliated with the School of Mathematics at Georgia Institute of Technology. Minshuo Chen is affiliated with the Department of Electrical and Computer Engineering at Princeton University. Hao Liu is affiliated with the Department of Mathematics at Hong Kong Baptist University. Tuo Zhao is affiliated with the School of Industrial and Systems Engineering at Georgia Institute of Technology. Email: \{jcheng328,tourzhao,wliao60\}@gatech.edu, mc0750@princeton.edu, haoliu@hkbu.edu.hk. This research is partially supported by NSF DMS 2012652 and NSF CAREER 2145167.}}
\date{}

\begin{document}
	\maketitle
	
	\begin{abstract}
		Label Shift has been widely believed to be harmful to the generalization performance of machine learning models. Researchers have proposed many approaches to mitigate the impact of the label shift, e.g., balancing the training data. However, these methods often consider the underparametrized regime, where the sample size is much larger than the data dimension. The research under the overparametrized regime is very limited. To bridge this gap, we propose a new asymptotic analysis of the Fisher Linear Discriminant classifier for binary classification with label shift. Specifically, we prove that there exists a phase transition phenomenon: Under certain overparametrized regime, the classifier trained using imbalanced data outperforms the counterpart with reduced balanced data. Moreover, we investigate the impact of regularization to the label shift: The aforementioned phase transition vanishes as the regularization becomes strong.
		\\
		
	\end{abstract}

	
\save{\section{Introduction}}
Label shift \cite{QuioneroCandela2009WhenTA} occurs predominantly 
in classification tasks, in domains like computer vision \citep{barandela2003restricted,xiao2010sun}, medical diagnosis \citep{grzymala2004approach,mac2002problem}, fraud detection \citep{philip1998toward} and others \citep{haixiang2017learning,radivojac2004classification}.
Label shift often stems from, for example, a non-stationary environment and a biased way that the training and test data sets are collected.

The basic assumption in label shift \citep{lipton2018detecting} is that, while the class prior changes, the conditional distributions of data within a class are maintained in training and testing. 
A well known special case of label shift is learning with imbalanced data  \citep{cao2019learning,zhao2021active} where the training are remarkably imbalanced due to some sampling bias, while the test data have a more balanced prior on the labels, e.g., uniform prior. It is commonly believed that training with imbalanced data can significantly undermine the overall performance of the trained classifiers \citep{japkowicz2002class,mazurowski2008training}. 

Typically, the test label distribution is unknown and many methods have been proposed to estimate the test priors \citep{Storkey09whentraining,chan05estimation}. When the test label distribution has been estimated, learning under label shift \citep{bishop1995nn,elkan2001costsensitive} reduces to the problem of resampling the training data. Common techniques include oversampling the minority class, downsampling the majority class \citep{branco2015survey,he2009learning,weiss2004mining}, and reweighting \citep{huang2016learning,wang2017learning}. For example, \citet{seiffert2009rusboost} integrate downsampling with boosting for classification with imbalanced training data. The performance (measured by the area under the ROC curve) improves compared to base line methods (e.g., adaboost). However, these balancing method comes with several drawbacks. When the data imbalancing is extreme, downsampling incurs significant loss of information, and oversampling can lead to over-fitting \citep{chawla2002smote,cui2019class}.  Reweighting methods tend to make the optimization of deep models difficult \citep{huang2016learning,cui2019class}.

The aforementioned findings primarily focus on the underparametrized regime, where the sample size is much larger than the data dimension. Recently, significant progress have been made in training overparametrized models, such as deep neural networks. Such progress have stimulated empirical and theoretical studies on overparametrized models, whose statistical properties surprisingly challenge the conventional wisdom. For example, the typical U-shaped bias-variance trade-off curve is complemented by the double descent phenomenon observed in various models (see more in the related work section, \citet{belkin2019reconciling,hastie2019surprises,bartlett2020benign,mei2019generalization}).  

In this paper, we study binary classification with label shift. The classifier is taken as Fisher Linear Discriminant Analysis (LDA, \citet{fisher1936use,bishop:2006:PRML}). 
We consider a Gaussian mixture data model, where $x\in \RR^p$
is the feature, $y \in \{0, 1\}$ is the label, and $x \given y$ is Gaussian distributed. Suppose $n$ training data are collected under certain prior with $n_{\ell}$ samples in class $\ell$ such that $n_0 +n_1 =n$. Specifically, data imbalance refers to the case where the majority and minority class priors do not match, i.e., $n_1/n_0 \neq 1$. The test prior on each class is denoted as $\pi_{\ell} = \PP(y=\ell)$ for $\ell =0,1$. We assume the test priors are known and different from the training priors. When the test priors are not known, we can still estimate them from the empirical label distributions in the test data.

{\bf Our contributions}. We provide a theoretical analysis on the performance of LDA under label shift, in both the under- and over-parametrized regime. We explicitly quantify the misclassification error in the proportional limit of $n \rightarrow \infty$ and $p/n_{\ell}  \rightarrow \gamma_{\ell}$ for $\ell=0,1$, where $\gamma_{\ell} > 0$ is a constant.
Our theory shows a peaking phenomenon when the sample size is close to the data dimension.

We demonstrate a {\it phase transition} phenomenon about data imbalance: The misclassification error exhibits different behaviors as the two-class ratio $n_1/n_0$ varies, depending on the value of $\gamma_0$.
In particular, when $\gamma_0$ is fixed and the ratio $n_1/n_0$ increases from $1$, we observe the following three phases:

\begin{itemize}
\setlength\itemsep{0em}

\item In the underparametrized regime (e.g., $\gamma_0 = 0.5$), the misclassification error first decreases then increases as $\gamma_1$ decays, yet the error decrease is marginal.
  
\item In the lightly overparametrized regime (e.g., $\gamma_0 = 2.5$), the misclassification error first increases then decreases as $\gamma_1$ decays. 
  
\item In the overparametrized regime (e.g., $\gamma_0 = 5$), the misclassification error first decreases then increases, and finally decreases again as $\gamma_1$ decays.
\end{itemize}
  
Such a phase transition suggests that LDA trained with imbalanced data can outperform the counterpart trained with reduced balanced data, in certain overparametrized regime.
  
Moreover, we investigate the impact of the $\ell_2$ regularization on the performance of LDA under label shift: The aforementioned phase transition vanishes when the regularization is sufficiently strong. While the phase transition persists when the regularization is weak.

{\bf Related work}. 
 In literature, many methods have been developed to handle classification under label shift. The sampling techniques include the informed upsampling \citep{liu2008exploratory,kubat1997addressing}, synthetic oversampling \citep{chawla2002smote}, cluster-based oversampling \citep{jo2004class}. Cost-sensitive methods \citep{sun2007cost, huang2016learning} use a cost matrix to represent the penalty of
classifying examples from one class to another. Examples are cost-sensitive decision trees \citep{maloof2003learning} and cost-sensitive neural networks \citep{kukar1998cost}. Kernel-based methods are developed in \citep{liu2006boosting,wang2010boosting,hong2007kernel}.
Despite the empirical success, there are limited theories about how the classification results are affected by data imbalance. 

Statistical properties of LDA has been well established in existing works (\citet{anderson1962introduction}, \citet[Section 10.2]{fukunaga2013introduction}, \citet{velilla2005consistency, zollanvari2011analytic, zollanvari2015generalized, sifaou2020high}). LDA with balanced training and test data is studied in \citet{raudys1998expected}  in the overparametrized case, while the assumption is more restrictive than ours, e.g., the feature vector is Gaussian with the identity covariance matrix in \citet{raudys1998expected}. In the asymptotic regime (i.e., $p, n \rightarrow \infty$), \citet{bickel2004some} show that, when $p/n\rightarrow \infty$, LDA tends to random guessing. Later, \citet{wang2018} consider the proportional scenario where $p / n \rightarrow \gamma \in (0, 1)$. Our theory is more general, and covers both $0<\gamma<1$ (underparametrization) and $\gamma>1$ (overparametrization). 
The misclassification error of Regularized LDA is analyzed in \citet{elkhalil2020large}. Our error analysis on LDA  can not be implied from \citet{elkhalil2020large} by taking the limit of the regularization parameter to $0$ since the covariance matrix is not invertible in the overparametrized case.
We note a parallel line of work studying LDA in sparsity constrained high-dimensional binary classification problems \citep{cai2011direct,shao2011sparse,mai2012direct}.

Our theory demonstrates a peaking phenomenon of LDA, which has been recognized in history \citep{hughes1968mean,duin1995small,hua2005optimal,sima2008peaking} and recently for neural networks \citep{belkin2019reconciling}. This phenomenon has been justified for linear regression \citep{hastie2019surprises,bartlett2020benign,belkin2020two,muthukumar2020harmless,bibas2019new}, random feature regression \citep{mei2019generalization}, logistic regression \citep{AMDDHBLC},  max-margin linear classifier \citep{GEMMLC}, and others \citep{xu2019number,derezinski2019exact,nakkiran2019more}. To our knowledge, we are the first to provide a theoretical justification of the peaking phenomenon for LDA under label shift.

The rest of the paper is organized as follows: Section \ref{sec:BinaryClassificationUsingLDA} introduces LDA; Section \ref{sec:PeakingPhenomenonOfLDA} presents an asymptotic analysis of the misclassification error for LDA, and the phase transition phenomenon under data imbalance; Section \ref{sec:RegularizationImpactOnMisclassification} presents the impact of regularization; Section \ref{sec:Experiments} presents real-data experiments; Section \ref{sec:proof} presents a proof of our main results; Section \ref{sec:ConclusionAndDiscussion} discusses binary classification with extremely imbalanced data and in the highly overparametrized regime. We also discuss future directions.

{\bf Notation}: Given a vector $v \in \RR^p$, we denote $\norm{v}_\Sigma=\sqrt{v^\top\Sigma v}$ for a positive definite matrix $\Sigma$. Given a matrix $M$, we denote $M^\dagger$ as its pseudo-inverse. Let $\Phi(\cdot)$ be the CDF of the standard normal distribution. For two random variables $X$ and $Y$, we denote $X\eqd Y$ as $X$ and $Y$ having the same distribution. For a sequence of random variables $\{X_n\}$, we denote $X_n\overset{\rm a.s.}{\longrightarrow}X$ as the almost sure convergence.

\save{\section{Binary Classification using LDA}} \label{sec:BinaryClassificationUsingLDA}

Binary classification aims at classifying an input feature $x\in \RR^p$ into two classes labeled by $y \in \{0,1\}$. A linear classifier achieves this goal by predicting the label based on a linear decision boundary in the form of $\beta^\top (x-\alpha) = b$ with $\alpha, \beta \in \RR^p$ and $b\in \RR$.

To be specific, a linear classifier gives the label $y$ of the feature $x$ by
\begin{align}
f^b_{\alpha,\beta}(x) = 
\begin{cases}
    0, \quad \text{if }\beta^\top (x-\alpha)>b,\\
    1, \quad \text{if }\beta^\top (x-\alpha)\leq b.
\end{cases}
\label{eq:linearclassifier}
\end{align}
Given the class priors $\pi_0 + \pi_1 = 1$, we determine $\alpha, \beta$ and $b$ by minimizing the misclassification error defined as
\begin{equation}\label{eq:miscls-err}
\begin{split}
    \mathcal{R} \left(f^b_{\alpha,\beta}\right)  & = \pi_0 \PP\left(f^b_{\alpha,\beta}(x)=1\given y=0\right) 
    \\  &+ \pi_1 \PP\left(f^b_{\alpha,\beta}(x)=0\given y=1\right).
\end{split}
\end{equation}
LDA approaches the binary classification problem by assuming that the conditional distribution of $x$ given label $y=0$ (resp. $y=1$) is Gaussian $\cN(\mu_0, \Sigma)$ (resp. $\cN(\mu_1, \Sigma)$). Accordingly, the optimal classifier in LDA (also known as the Bayes rule) takes 
\begin{equation}
\label{eq:alphabetastar}
\alpha^* = \frac{\mu_0+\mu_1}{2}, ~~\beta^* = \Sigma^{-1} \left(\mu_0-\mu_1\right), ~~ b^* = \ln \frac{\pi_1}{\pi_0}.
\end{equation}
The decision boundary $\beta^*$ coincides with the Fisher linear discriminant rule \citep{fisher1936use}, which maximizes the ratio of between-class variance and within-class variance:
\begin{equation}
  \argmax_{\beta \in \RR^p}~   \frac{ (\beta^\top \mu_1- \beta^\top \mu_0)^2}{\beta^\top  \Sigma \beta}.
  \label{eqmaxfisher}
\end{equation}
The optimal ratio in \eqref{eqmaxfisher} at $\beta^*$ is defined as the Signal-to-Noise Ratio (SNR), i.e., $\textrm{SNR} = \norm{\beta^*}_\Sigma^2 = (\beta^*)^{\top} \Sigma \beta^*$.

In practice, we receive $n_0$ and $n_1$ i.i.d training data points from class $0$ and $1$, respectively. We denote the per-class data points as $\{x^\ell_i\}_{i=1}^{n_\ell}$ for $\ell = 0, 1$. The total number of samples is $n = n_0 + n_1.$ We obtain the empirical Fisher linear discriminant classifier $f_{\hat\alpha,\hat \beta}^{\hat b}$ with
\begin{align}\label{eq:estimator}
\hat{\alpha} = \frac{\hat{\mu}_0 + \hat{\mu}_1}{2}, \quad \hat{\beta} = \hat{\Sigma}^{\dagger}(\hat{\mu}_0 - \hat{\mu}_1), \quad \hat{b} = \ln \frac{n_1}{n_0},
\end{align}
where $\hat{\mu}_0, \hat{\mu}_1$ and $\hat{\Sigma}$ are empirical estimators of $\mu_0, \mu_1, \Sigma$:
\begin{align*}
\hat{\mu}_\ell = \frac{1}{n_\ell} \sum_{i=1}^{n_\ell} x_i^\ell, ~~ \hat{\Sigma} = \frac{1}{n-2} \sum_{\ell = 0, 1}\sum_{i=1}^{n_\ell} (x_i^\ell - \hat{\mu}_\ell)(x_i^\ell - \hat{\mu}_\ell)^\top.
\end{align*}

\save{\section{Phase Transition of LDA under Label Shift}}\label{sec:PeakingPhenomenonOfLDA}

In this section, we present our main results on the misclassification error analysis of LDA, which covers both the under- and over-parametrized regime. 
\save{\subsection{Error Analysis of LDA}}\label{sec:peaklda}
We first introduce a data model for our theoretical analysis.
\begin{assumption}\label{asm:training data and test data} For both the training and the test data, the conditional distribution of $x$ given $y=\ell$ is Gaussian, i.e., $$x \given(y=\ell) \sim \cN(\mu_\ell,\Sigma), \quad\textrm{for} \quad \ell=0,1.$$ The training data $\{x_i^\ell\}_{i=1}^{n_\ell}$ are i.i.d. sampled for class $\ell = 0, 1$, respectively. The test data have priors $\pi_0$, $\pi_1$ such that $\pi_0+\pi_1=1$. \end{assumption}

Assumption \ref{asm:training data and test data} allows arbitrarily imbalanced training data with $n_0 \neq n_1$ where the test label distribution can be different from the training data.
Under Assumption \ref{asm:training data and test data}, we prove an asymptotic behavior of the misclassification error of LDA in the limit of $n,p\to\infty$.

\begin{theorem}
\label{theorem1}
Let $\gamma_0,\gamma_1$ and $\Delta^2$ be positive constants and set $\gamma=\frac{\gamma_0\gamma_1}{\gamma_0+\gamma_1}$.
Under Assumption \ref{asm:training data and test data}, we let $n,p\to\infty$ with
\begin{equation*}
p/n_0 \to \gamma_0,\ p/n_1 \to \gamma_1,\ p/n \to \gamma, \text{ and } \norm{\beta^*}_\Sigma^2\to\Delta^2.
\end{equation*} 
Then the misclassification error $\cR(f_{\alphahat,\betahat}^{\hat b})$ converges to a limit almost surely when $\gamma \in (0,1)\cup (1,\infty)$, i.e.,
\begin{equation}\label{eq:theorem-underparametrized}
    \Rcal (f_{\alphahat, \betahat}^{\hat b}) \overset{\rm a.s.}{\rightarrow} \sum_{\ell=0,1}\pi_\ell \Phi\left(\frac{\frac{g(\gamma_0, \gamma_1, \ell)}{1-\gamma}+(-1)^\ell\ln{\frac{\gamma_0}{\gamma_1}}}{k(\gamma_0, \gamma_1)\frac{1}{(1-\gamma)^{3/2}}}\right)
\end{equation}
for $0 < \gamma < 1$, and 
\begin{equation}\label{eq:theorem-overparametrized}
    \mathcal{R}(f_{\alphahat, \betahat}^{\hat b}) \overset{\rm a.s.}{\rightarrow}\sum_{\ell=0,1}\pi_\ell \Phi\left(\frac{\frac{g(\gamma_0, \gamma_1, \ell)}{\gamma(\gamma-1)}+(-1)^\ell\ln{\frac{\gamma_0}{\gamma_1}}}{k(\gamma_0, \gamma_1)\frac{1}{(\gamma-1)^{3/2}}}\right)
\end{equation}
for $\gamma > 1$, where
\begin{align*} 
g(\gamma_0, \gamma_1, \ell) &= -\frac{1}{2}(\Delta^2 + (-1)^\ell (\gamma_0 - \gamma_1)) \\
 k(\gamma_0, \gamma_1) &= \sqrt{\Delta^2 + \gamma_0 + \gamma_1}.
 \end{align*}
\end{theorem}

Theorem \ref{theorem1} is proved in Section \ref{appendix:pfthm1}.
Theorem \ref{theorem1} demonstrates a peaking phenomenon of LDA. For simplicity, we illustrate this peaking phenomenon for balanced training ($\gamma_0=\gamma_1$) and test data ($\pi_0 = \pi_1 = 0.5$). In this case, the limit of $\mathcal{R}(f_{\alphahat, \betahat}^{\hat b})$ can be simplified to
\begin{equation*}
    \Rcal(f_{\alphahat,\betahat}^{\hat b})\overset{\rm a.s.}{\longrightarrow}    \begin{cases}
        \Phi\left(-\frac{\Delta^2 \sqrt{1-\gamma}}{2\sqrt{\Delta^2 + 4\gamma}}\right), &\ 0<\gamma<1,\\
        \Phi\left(-\frac{\Delta^2 \sqrt{\gamma-1}}{2\gamma \sqrt{\Delta^2 + 4\gamma}}\right), &\ \gamma>1.
    \end{cases}
     \label{eq.R.balance}
\end{equation*}
Figure \ref{fig:limiting_risk} (left panel) shows the asymptotic misclassification error as a function of $\gamma$ with various SNR. 

A peak occurs as $\gamma$ approaches $1$, when the training sample size is approximately equal to the data dimension. In the {\it underparametrized} regime ($0<\gamma<1$), the misclassification error increases with respect to $\gamma$. In the {\it overparametrized} regime ($\gamma>1$), the misclassification error has a local minimum such that the error first decreases and then increases. This peaking phenomenon exists for all levels of SNR.

 We interpret the peaking phenomenon as an interplay between the conditioning of the covariance matrix and the variance in the statistical estimation of the means and the covariance matrix. With balanced training data, the training model and the test model are the same so there is no model mismatch.
  In the underparametrized regime ($0<\gamma<1$), the covariance matrix is full rank. In this case,  variance dominates the estimation error, and variance decreases as the sample size increases, since a larger number of samples yield better estimations of $\mu_0, \mu_1$, and $\Sigma$.   Therefore, the misclassification error decreases as $\gamma$ decreases. In the overparametrized regime ($\gamma>1$), the covariance matrix is rank-deficient.  According to Bai-Yin theorem \citep{bai1993}, the condition number of $\Sigma^\dagger$ is proportional to $(1-\sqrt{1/\gamma})^{-2}$, which decreases as $\gamma$ increases from $1$. As a result, the misclassification error decreases as $\gamma$ increases from $1$. When $\gamma$ further increases, variance dominates the error due to limited number of samples, so the error increases again. This explains the local minimum when $\gamma>1$.
  
\begin{figure}[ht]
    \centering
    \includegraphics[width=0.9\textwidth]{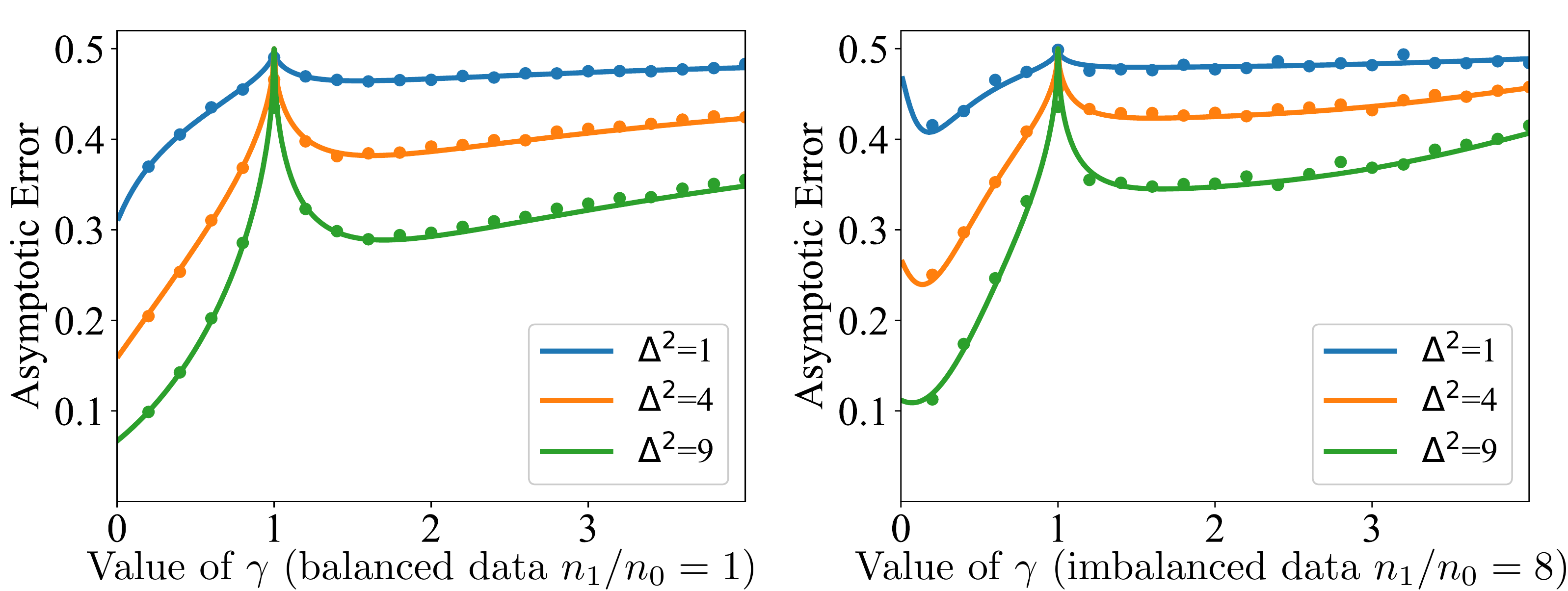}
    \caption{Misclassification error of LDA as a function of $\gamma$ when $n_1/n_0=1$ (left panel) and $n_1/n_0=8$ (right panel), with balanced test data ($\pi_0 = \pi_1 = 0.5$), for various levels of asymptotic SNR. The solid curve represents the theoretical error given by Theorem \ref{theorem1}, and the scatters denote synthetic data errors, with $n=200,p=\lceil \gamma n \rceil$.}
    \label{fig:limiting_risk}
\end{figure}
  More interestingly, under the label shift with imbalanced training data ($\gamma_0 \neq \gamma_1$) and balanced test data ($\pi_0=\pi_1=0.5$), there are trade-offs among three factors: (1) the conditioning of the pseudo-inverse of the covariance matrix; (2) the variance in the statistical estimation of the means and the covariance matrix; (3) an additional model mismatch. The misclassification error exhibits intricate behaviors with multiple local minima in the error curve. See Figure \ref{fig:limiting_risk} (right panel) for an example with $n_1/n_0=8$.
  
\save{\subsection{Phase Transition under Label Shift}}\label{sec:phaselda}
\vspace{-1mm}
A critical question for classification under label shift is: When the class priors vary between the training data and the test data, is it beneficial to correct the training data distribution?

In order to evaluate the overall performance of classifier trained with imbalanced data, we usually consider a special case of label shift with balanced test data ($\pi_0=\pi_1=0.5$). The question above is reduced to whether it is beneficial to downsample the majority class. Theorem \ref{theorem1} suggests an interesting dichotomy to the question above. Specifically, we fix $\gamma_0$ and investigate how the two-class ratio $n_1/n_0 \in [1, 10]$ affects the performance of LDA. We identify three distinct behaviors of LDA depending on the value of $\gamma_0$. We demonstrate the three behaviors in Figure \ref{fig:lda_asym_risk_imbalanced123}:

\noindent$\bullet$ {\it Behavior I:} When $\gamma_0=0.5$, i.e., the class of $\ell = 0$ is underparametrized, the misclassification error first decreases and then increases as a function of $n_1/n_0$; 

\noindent$\bullet$ {\it Behavior II:} When $\gamma_0 =2.5$, i.e., the class of $\ell = 0$ is slightly overparametrized, the misclassification error first increases and then decreases as a function of $n_1/n_0$;

\noindent$\bullet$ {\it Behavior III:} When $\gamma_0 = 5$, i.e., the class of $\ell = 0$ is overparametrized, the misclassification error first decreases and then increases, and finally decreases again, as a function of $n_1/n_0$. 

We obtain rich insights on training with imbalanced data from Figure \ref{fig:lda_asym_risk_imbalanced123}. When the data imbalance is moderate, for example $n_1/n_0 \in [1, 3]$, training with imbalanced data can outperform the counterpart of using reduced balanced data as in {\it Behavior I} and {\it Behavior III}. Nonetheless, the improvement in {\it Behavior I} is only marginal. On the contrary, {\it Behavior II} indicates that downsampling the majority class improves the performance of LDA.

\begin{figure}[h]
    \centering
    \includegraphics[width=0.9\textwidth]{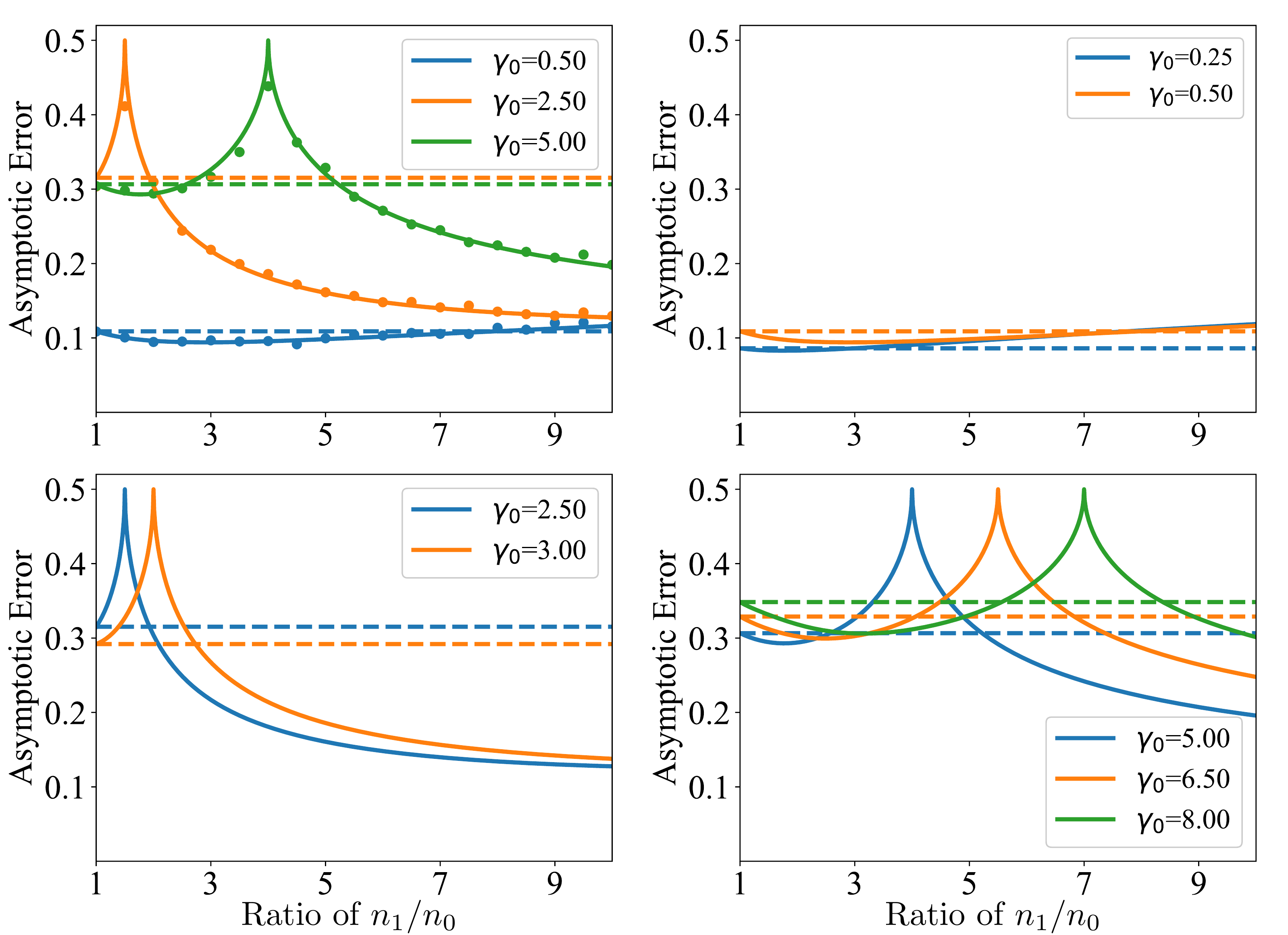}
    \caption{We demonstrate three behaviors of the misclassification error of LDA as a function of $n_1/n_0$ for various $\gamma_0$ and a fixed asymptotic SNR $\Delta^2=9$. The solid curve represents the error based on our theoretical analysis in Theorem \ref{theorem1}. The scatters denote synthetic data errors, with $n_0=40, p=\lceil \gamma_0n_0 \rceil$. Upper left panel: Plots of an underparametrized ($\gamma_0 =0.5$), a slightly overparametrized ($\gamma = 2.5$) and an overparametrized ($\gamma = 5$) case. Upper right panel: Behavior I in the underparametrized cases; Bottom Left panel: Behavior II in the slighly overparametrized cases. Bottom right panel: Behavior III in the overparametrized cases.}
    \label{fig:lda_asym_risk_imbalanced123}
\end{figure}

As the data imbalance becomes more significant, for example, $n_1/n_0 \in [8, 10]$, {\it Behavior II} and {\it III} both indicate that downsampling the majority class hurts the performance. Such behavior is expected, since the downsampling incurs severe information loss.

Theorem \ref{theorem1} also characterizes the misclassification error when $n_1/n_0$ is extremely large. We can check that the misclassification error converges to $0.5$ as $n_1/n_0 \rightarrow \infty$, regardless of the value of $\gamma_0$ (see Appendix \ref{appendix:pfphasetransition}). This indicates that extreme label shift renders the trained classifier suffering from the model mismatch. Nonetheless, in such an extreme imbalanced case, the minority group is prone to be outliers, and detection of outliers is also of great interest.

In the sequel, we formally characterize three phases corresponding to the aforementioned different behaviors. We explicitly identify two phase transition knots $\gamma_a$ and $\gamma_b$ (derived in Appendix \ref{appendix:pfphasetransition}):
\begin{equation*}
    \gamma_a = 2 \quad \textrm{and} \quad \gamma_b = \frac{1}{8}(12-\Delta^2+\sqrt{\Delta^4+40\Delta^2+144}).
\end{equation*}
We claim three phases depending on the value of $\gamma_0$.

\noindent$\bullet$ {\bf Phase I}: If $\gamma_0\in(0,\gamma_a)$, the misclassification error has {\it Behavior I};

\noindent$\bullet$ {\bf Phase II}: If $\gamma_0\in(\gamma_a,\gamma_b)$, the misclassification error has {\it Behavior II};

\noindent$\bullet$ {\bf Phase III}: When $\gamma_0\in(\gamma_b,\gamma_c)$ for some $\gamma_c>0$,  the misclassification error has {\it Behavior III}.

We observe that the first transition $\gamma_a$ appears at the exact parametrized case, i.e., $p = n$. The second transition $\gamma_b$ depends on the SNR and is always larger than $\gamma_a$. 

We remark that in {\bf Phase III}, we cut off $\gamma_0$ at some threshold $\gamma_c$. If $\gamma_0$ is extremely large, i.e., the problem is highly overparametried, we can observe a fourth behavior on the misclassification error. In fact, the misclassification error has multiple local maxima, reflecting a complex interaction between the limited information in the training data and the mismatch of the training and test model. We discuss the highly overparametried regime in Section \ref{sec:ConclusionAndDiscussion}.

\save{\section{Regularization Impact on LDA}}\label{sec:RegularizationImpactOnMisclassification}
In machine learning, regularization is commonly used to stabilize the computation and improve the generalization performance. 
In this section, we study regularized LDA \citep{friedman1989regularized,guo2007regularized} and analyze its asymptotic misclassification error. 

\save{\subsection{Error Analysis of Regularized LDA}}\label{sec:rldaerror}
When an $\ell_2$ regularization term is added on $\beta$, we consider the following optimization problem based on \eqref{eqmaxfisher}:
\begin{equation*}
    \argmin_{\beta \in \RR^p}~ -(\beta^\top \mu_1- \beta^\top \mu_0)^2 + \lambda \norm{\beta}_2^2, \ \  \textrm{s.t.} \ \ \beta^\top  \Sigma \beta = 1.
\end{equation*}
This gives rise to an optimal solution $\beta^*_\lambda$. For simplicity, we formalize it in an equivalent form as
\begin{equation}\begin{split}
    \beta^*_\lambda = (\Sigma + \lambda I )^{-1} \big({\mu}_0-{\mu}_1\big).
\end{split}\end{equation}
We denote the regularized Fisher linear discriminant classifier by $f_{\alpha^*,\beta_\lambda^*}^{b^*}$ where $\alpha^*$ and $b^*$ are given in \eqref{eq:alphabetastar}.

The empirical counterpart of the regularized LDA is given by $f_{\hat\alpha,\hat\beta_\lambda}^{\hat b}$, where the empirical parameters $\hat\alpha$, $\hat\beta_\lambda$ and $\hat b$ are computed through $\muhat_0, \muhat_1$ and $\Sigmahat$ according to \eqref{eq:estimator} and
$$
   \hat \beta_\lambda = (\Sigmahat + \lambda I )^{-1} \big({\muhat}_0-{\muhat}_1\big).
$$

Similar to Theorem \ref{theorem1}, we prove an asymptotic behavior of the misclassification error of the regularized LDA.

\begin{theorem}\label{theorem2regularized}
Let $\gamma_0,\gamma_1$ and $\Delta^2$ be positive constants and set $\gamma=\frac{\gamma_0\gamma_1}{\gamma_0+\gamma_1}$.
Under Assumption \ref{asm:training data and test data}, we let $n,p\to\infty$ with
\begin{equation*}
    p/n_0 \to \gamma_0,\  p/n_1 \to \gamma_1,\  p/n \to \gamma, \text{ and } \norm{\beta^*}_\Sigma^2 \to \Delta^2.
\end{equation*} The misclassification error $\cR(f^{\hat{b}}_{\alphahat, \widehat{ \beta}_\lambda})$
converges to a limit almost surely when $\gamma \in (0,1) \cup (1,\infty)$, i.e.,
\begin{equation*}
\cR(f^{\hat{b}}_{\alphahat, \widehat{ \beta}_\lambda}) \overset{\rm a.s.}{\rightarrow}\sum_{\ell=0,1} \pi_\ell  \Phi\left(\frac{g(\gamma_0, \gamma_1, \ell)m(-\lambda)+(-1)^\ell\ln{\frac{\gamma_0}{\gamma_1}}}{k(\gamma_0, \gamma_1)\sqrt{m'(-\lambda)}}\right),
\end{equation*}
where $g(\gamma_0, \gamma_1, \ell)$ and $k(\gamma_0,\gamma_1)$ are defined in Theorem \ref{theorem1}, and $m(\lambda) = \int 1/(s - \lambda) dF_\lambda(s)$ with $F_\lambda$ denoting the Marchenko-Pastur law.
\end{theorem}

Theorem \ref{theorem2regularized} is proved in Appendix \ref{appendixtheorem2regularized}. 
Theorem \ref{theorem2regularized} implies that regularization has a smoothing effect on the peaking phenomenon. For simplicity, we show this smoothing effect for balanced training ($\gamma_0=\gamma_1$) and test data ($\pi_0 = \pi_1 = 0.5$).

Figure \ref{fig:rlda_asym_risk_lam} shows the asymptotic misclassification error of regularized LDA as a function of $\gamma$. 
When the regularization is weak, e.g., $\lambda = 10^{-4}$, we observe a similar peaking phenomenon as in Figure \ref{fig:limiting_risk}, while the peak is lower than that in Figure \ref{fig:limiting_risk}. Compared to the unregularized classifier $f^{\hat b}_{\hat \alpha, \hat \beta}$, regularization improves the conditioning of the estimated covariance matrix, i.e., $\hat{\Sigma} + \lambda I$ is never singular, which in turn mitigates the performance degradation when $\gamma \approx 1$.

\begin{figure}[htb!]
    \centering
    \includegraphics[width=0.9\textwidth]{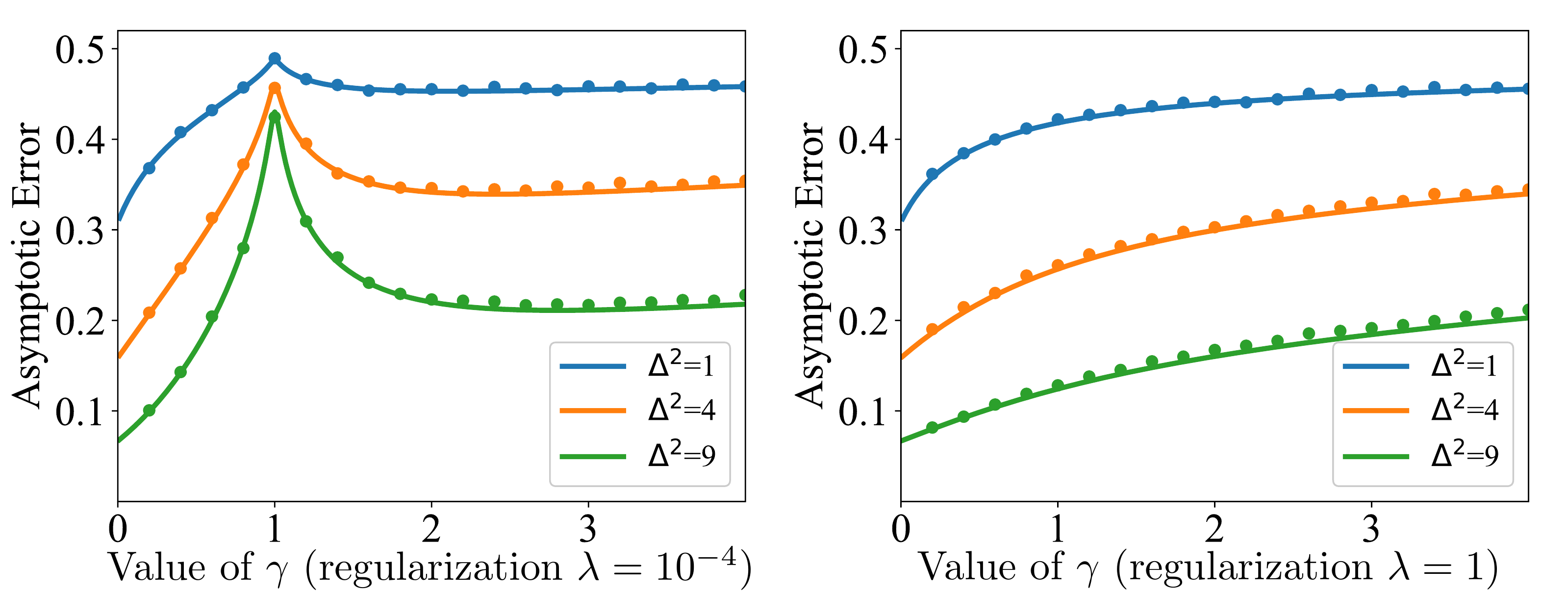}
    \caption{Misclassification error of regularized LDA as a function of $\gamma$ for various levels of asymptotic SNR. The left panel shows the impact of weak regularization $\lambda = 10^{-4}$, and the right panel uses a strong regularization $\lambda = 1$. The solid line represents the theoretical error given by Theorem \ref{theorem2regularized}, and the scatters denote synthetic data errors, with $n=200,p=\lceil \gamma n \rceil$.}
    \label{fig:rlda_asym_risk_lam}
\end{figure}

\begin{figure}[htb!]
\begin{minipage}[c]{\textwidth}
    \centering
        \includegraphics[width=0.45\textwidth]{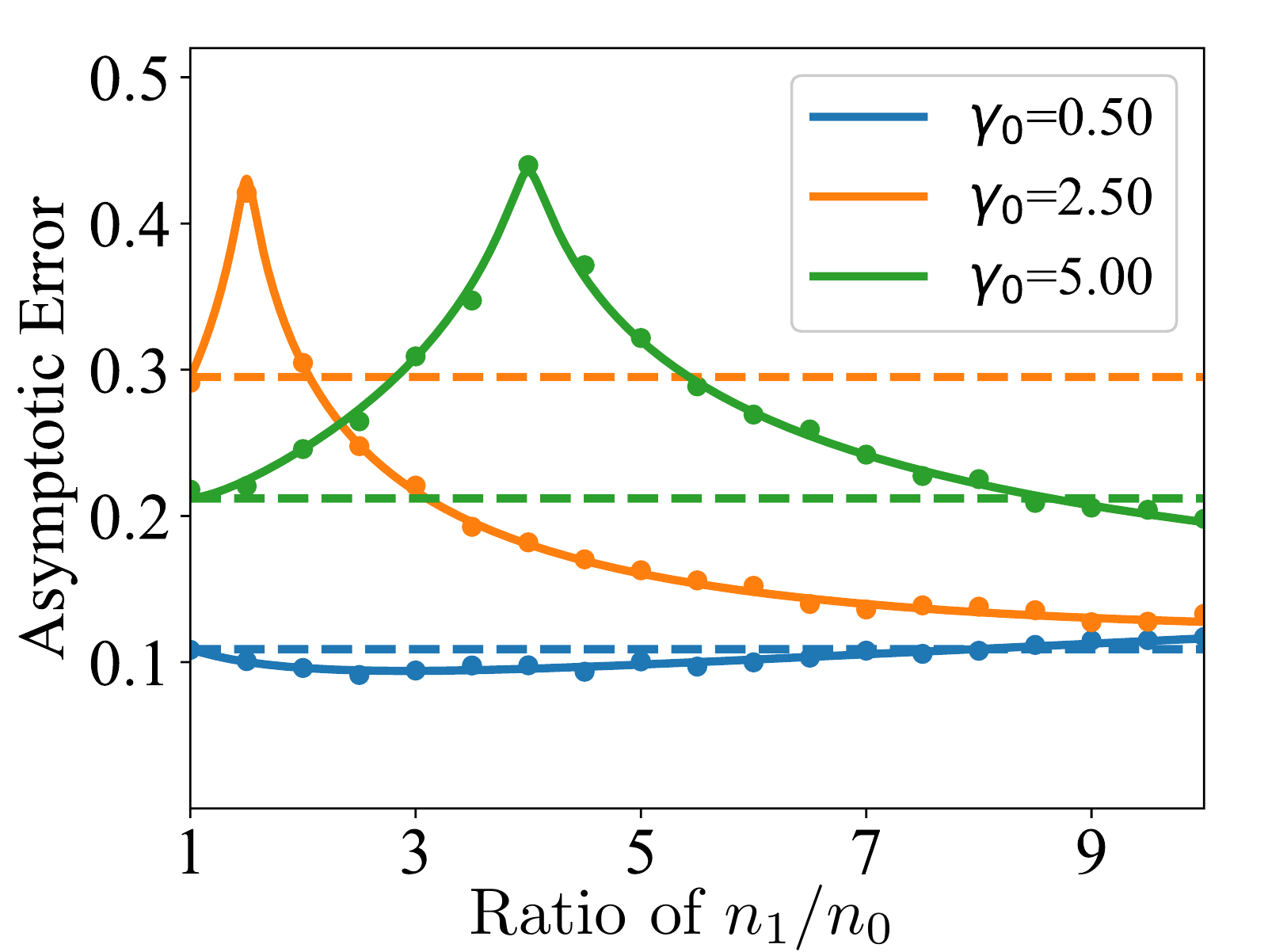}
        \includegraphics[width=0.45\textwidth]{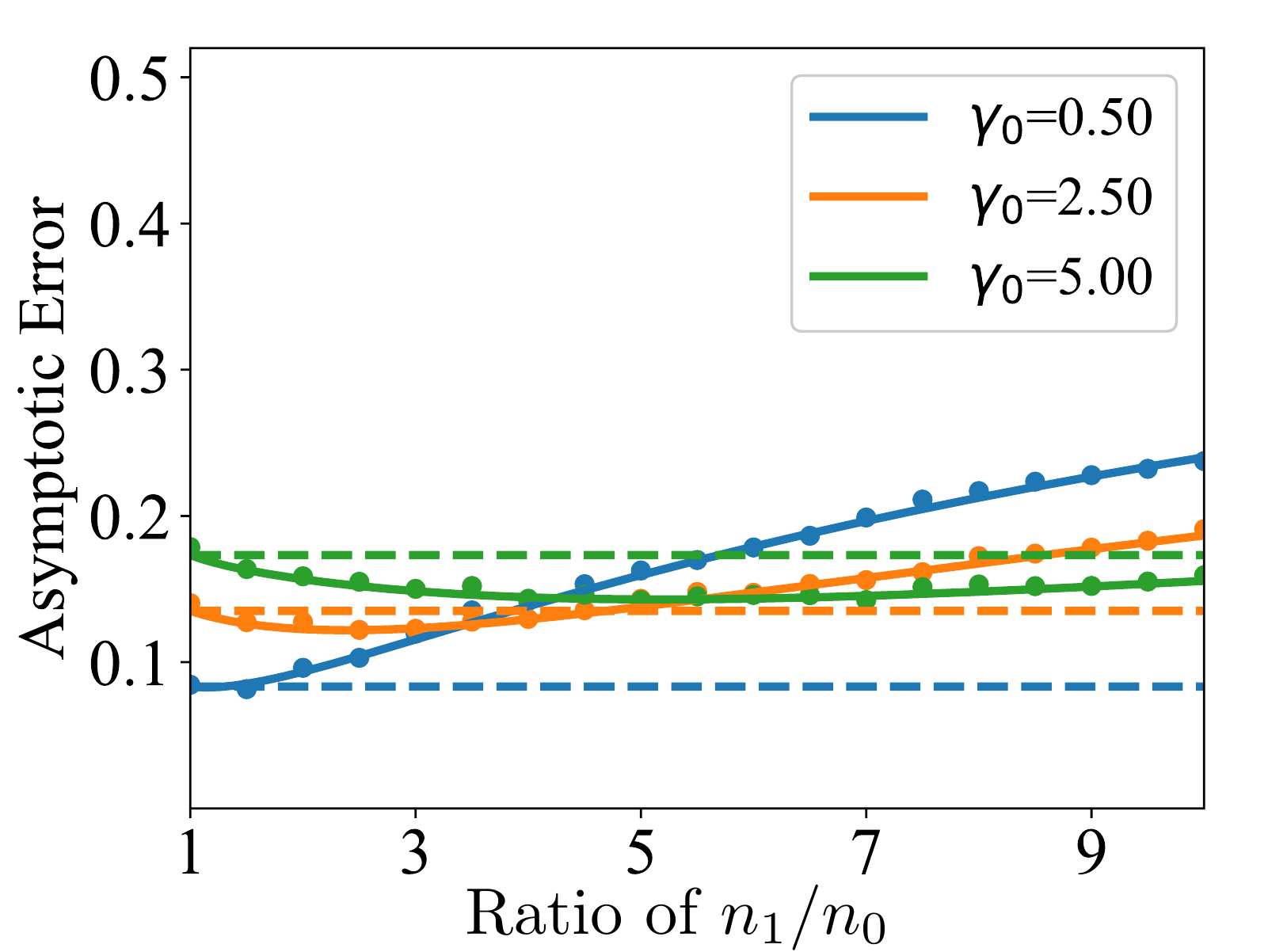}
\end{minipage}

    \caption{We demonstrate three behaviors of the misclassification error of regularized LDA as a function of $n_1/n_0$ for various $\gamma_0$ when we use a weak regularization with $\lambda = 10^{-4}$ (left panel). They disappears with a strong regularization $\lambda = 1$ (right panel). We fix the asymptotic SNR $\Delta^2=9$. The scatters denote synthetic data errors, with $n_0=40, p=\lceil \gamma_0n_0 \rceil$.}
    \label{fig:rlda_asym_risk_imbalanced123}
    
\end{figure}

When the regularization is strong, e.g., $\lambda = 1$, the peaking phenomenon disappears in Figure \ref{fig:rlda_asym_risk_lam} (right panel). In this case, the matrix $\hat{\Sigma}+\lambda I$ is always well-conditioned. The error is dominated by the variance in the statistical estimation of the means and the covariance matrix. As a result, the error increases as $\gamma$ increases. We remark that proper regularization greatly reduces the misclassification error when the problem is approximately exactly parametrized ($\gamma \approx 1$).

\save{\subsection{Phase Transition of Regularized LDA}}\label{sec:rldaphase}

In this section, we study the impact of the regularization on the phase transition phenomenon discussed in Section \ref{sec:phaselda}. We discuss the impact of weak and strong regularization separately, as they lead to very different behaviors.

When the regularization is weak, e.g., $\lambda = 10^{-4}$, we observe a similar phase transition phenomenon as in Section \ref{sec:phaselda}. We depict the misclassification error curves as a function of $n_1/n_0$ in Figure \ref{fig:rlda_asym_risk_imbalanced123}.

When the regularization is strong, e.g., $\lambda = 1$, the phase transition phenomenon disappears (See a formal justification in Appendix \ref{appendix:rldaphasetransition}). Figure \ref{fig:rlda_asym_risk_imbalanced123} shows that the asymptotic misclassification error of regularized LDA as a function of $n_1/n_0$ for various $\gamma_0$. We observe that the error curve consistently first decreases and then increases. In this case, the matrix $\hat{\Sigma}+\lambda I$ is always well-conditioned. Therefore the misclassification error is the consequence of the trade-off between two factors: 1) the model mismatch and 2) the variance in the statistical estimation of the means and the covariance matrix.

\save{\section{Real-Data Binary Classification}}\label{sec:Experiments}

We connect our theoretical findings to real-data binary classification tasks. We consider Neural Network (NN) classifiers for the MNIST (CC-BY 3.0) and CIFAR-10 (MIT) datasets \citep{lecun1998gradient, krizhevsky2009learning}. We focus on the overparametried regime, which is the working regime for neural networks.

\textbf{MNIST dataset and NN classifier}. The MNIST dataset consists of handwritten digits of resolution $28 \times 28$. We train a  neural network classifier with one hidden layer to distinguish digits $3$ and $8$. We vary the number of hidden units in $\{32, 64, 128\}$. The activation function is ReLU, i.e., $\sigma(\cdot) = \max\{\cdot, 0\}$. We use Adam \citep{kingma2014adam} for training, with default hyperparameters in Pytorch.

\textbf{CIFAR-10 dataset and NN classifier}. The CIFAR-10 dataset consists of RGB images of resolution $32 \times 32$ from $10$ categories. We pick two similar categories, e.g., horse v.s. deer, for binary classification. We downsample the data to a resolution of $16 \times 16$. We also train an NN classifier with one hidden layer. The number of neurons in the hidden layer is $10$, and the activation function is ReLU. We use momentum SGD for training, with momentum coefficient $0.9$ and learning rate $0.001$.

\begin{figure}[htb!]
    \centering
    \includegraphics[width=0.9\textwidth]{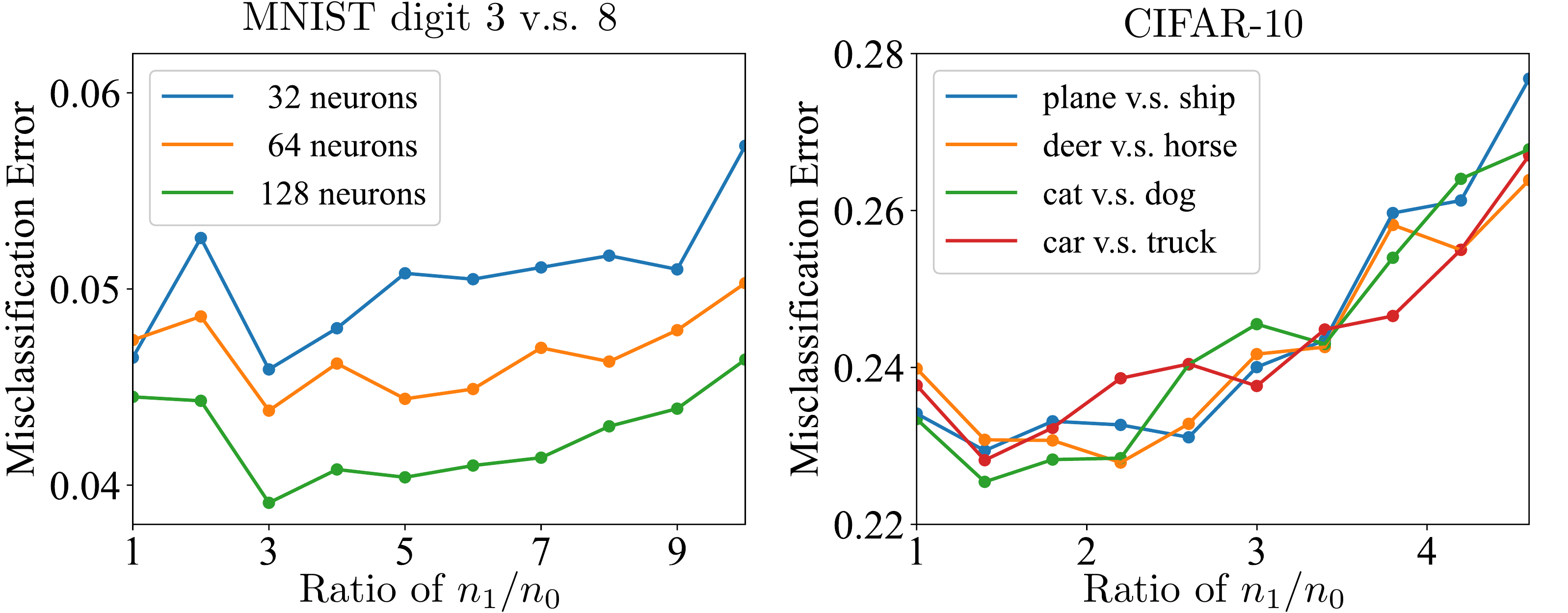}
    \caption{Misclassification error of neural network classifiers on MNIST and CIFAR-10 for binary classification.}
    \label{fig:realdata}
\end{figure}
In both tasks, during training, we fix the number of samples from one category ($500$ in MNIST and $1000$ in CIFAR-10), and vary the samples in the other category. The total number of training epochs is $50$ for MNIST and $20$ for CIFAR-10. After training, the classifier is tested on a balanced test set, which consists of $900$ samples from each class in MNIST, and $1000$ samples per-class in CIFAR-10. The test error is averaged over $10$ independent runs for MNIST and $5$ independent runs for CIFAR-10 with random seeds.

\textbf{Result}. The misclassification error in MNIST and CIFAR-10 as a function of $n_1/n_0$ is plotted in Figure \ref{fig:realdata}. In both the MNIST and CIFAR-10 experiments, the neural network is overparametrized. Therefore, we expect the misclassification error exhibits {\it Behavior III} in Figure \ref{fig:lda_asym_risk_imbalanced123}. This is corroborated in Figure \ref{fig:realdata} as the misclassification error first decreases and then increases as $n_1/n_0$ grows.

More importantly, these experiments consistently indicate that in the overparametrized regime, downsampling the majority class may hurt the model performance (cf. $n_1/n_0 \approx 3$ in MNIST and $n_1/n_0 \approx 1.5$ in CIFAR-10). Meanwhile, when the data imbalance is relatively severe, downsampling the majority class can be beneficial, due to its mitigation on label shift between training and testing.

\save{\section{Proof of Theorem \ref{theorem1}}} \label{sec:proof}

In this Section, we prove our main Theorem \ref{theorem1}. The proof for Theorem \ref{theorem2regularized} about the regularized LDA is similar to the proof of Theorem \ref{theorem1}. The proof of Theorem \ref{theorem2regularized} is given in Appendix \ref{appendixtheorem2regularized}.

\subsection{Lemmas to be used for the proof of Theorem \ref{theorem1}}
\label{subsec:lemma}
Our analysis relies on change of variables to exploit the independence between the sample mean estimator and the covariance estimator.

\begin{lemma}[Independence between sample mean and sample covariance]\label{lemma:independent} 
Let $x_i \overset{\rm i.i.d.}{\sim} \cN(\mu, \Sigma)$ for $i = 1, \dots, n$, be samples from a Gaussian distribution. We denote $\hat{\mu} = \frac{1}{n} \sum_{i=1}^n x_i$ as the estimator of the sample mean, and $\hat{\Sigma} = \frac{1}{n-1} \sum_{i=1}^n (x_i - \hat{\mu})(x_i - \hat{\mu})^\top$ as the estimator of the sample covariance. Then $\hat{\mu}$ and $\hat{\Sigma}$ are independent.
\end{lemma}

We also utilize the asymptotic characterization of the spectrum of Wishart matrix.
\begin{lemma}[Isotropicity of wishart matrix]\label{lemma3:isotropicity}
Assume $\frac{Z^\top Z}{n-2} \sim \mathcal{W}(I_p , n-2)$. For any vector $z \in \RR^p$ independent of $Z$, we have
\begin{align}\label{apeq:wishart}
{z}^\top \left(\frac{Z^\top Z}{n-2} \right)^{\dagger}{z} \overset{d}{=} \frac{1}{p}\norm{z}^2_2 \tr\left(\left(\frac{Z^\top Z}{n-2}\right)^\dagger\right).
\end{align}
\end{lemma}

Lemma \ref{lemma:independent} and \ref{lemma3:isotropicity} are proved in Appendix \ref{appendix:pflma}.
We next present some helper lemmas.

\begin{lemma}\label{lemma:trace}
Given a matrix $Z$ with i.i.d. standard normal distributed entries, we have 
\begin{align}\label{apeq:trace}
\tr\left(\left(Z^\top Z\right)^\dagger\right) =\tr\left(\left(ZZ^\top \right)^\dagger\right).
\end{align}
\end{lemma}

\begin{lemma}[Strong law of large numbers]\label{lemma:SLLN}
Assume $z\sim\mathcal{N}(0,I_p)$ and $\mu_d$ is a non-random p-dimensional vector such that $\norm{\mu_d}_2^2 \overset{\rm a.s.}{\longrightarrow} \Delta^2$ and $n_\ell$ satisfies Binomial distribution $B(n_0+n_1,\pi_\ell)$ then we have
\begin{equation*}\begin{split}
    \frac{1}{\sqrt{p}}\mu_d^\top z \overset{\rm a.s.}{\longrightarrow}0,\quad
    \frac{1}{p}z^\top z \overset{\rm a.s.}{\longrightarrow} 1,\quad
    \frac{n_\ell}{n_0+n_1} \overset{\rm a.s.}{\longrightarrow} \pi_\ell.
\end{split}\end{equation*}
\end{lemma}

\begin{lemma}[Marchenko-Pastur law]\label{lem.MPLaw}
Let $F_{\gamma}$ be the Marchenko-Pastur (MP) law of $0<\gamma<1$. Then for any real number $\zeta\leq 0$, the Stieltjes transform of MP law and its derivative at $\zeta$ are given as
\begin{align}
    m(\zeta)&=\int \frac{1}{s-\zeta}dF_{\gamma}(s)=\frac{1-\gamma-\zeta-\sqrt{(\zeta-\gamma-1)^2-4\gamma}}{2\gamma\zeta}
    \label{eq.MP.m}
\end{align}
and 
\begin{align}
    \frac{d}{d\zeta}m(\zeta)&=\int \frac{1}{(s-\zeta)^2}dF_{\gamma}(s) =\frac{2\gamma(\gamma-1)-\frac{2\gamma\left(\zeta(\gamma+1) -(\gamma-1)^2\right)}{\sqrt{(\zeta-\gamma-1)^2 -4\gamma}} }{4\gamma^2\zeta^2}.
    \label{eq.MP.dm}
\end{align}

In particular, we have
\begin{align}
    &m(0)=\frac{1}{1-\gamma}  \quad \mbox{ and } \quad \frac{d}{d\zeta}m(0)=\frac{1}{(1-\gamma)^3}. \nonumber
\end{align}
\end{lemma}

Lemma \ref{lemma:trace}, \ref{lemma:SLLN} and \ref{lem.MPLaw} are proved in Appendix \ref{subsec:lemmaproof2}.

\save{\subsection{Proof of Theorem \ref{theorem1}}}\label{appendix:pfthm1}

\begin{proof}[Proof of Theorem \ref{theorem1}]

To begin with, we recall the misclassification error of binary classification in \eqref{eq:miscls-err}. Substituting the Fisher linear discriminant classifier $f^{\hat{b}}_{\hat{\alpha}, \hat{\beta}}$ (with $\hat{\alpha}, \hat{\beta}, \hat{b}$ given in \eqref{eq:estimator}) and prior $\pi_0 = \pi_1 = 1/2$ into \eqref{eq:miscls-err}, we derive
\begin{align}
\cR(f^{\hat{b}}_{\hat{\alpha}, \hat{\beta}})& = \frac{1}{2}\PP \left(\hat \beta ^\top (x-\hat\alpha) \leq \ln\frac{n_1}{n_0} \given x\sim \cN(\mu_0,\Sigma)\right) \nonumber \\
& \quad + \frac{1}{2} \PP \left(\hat \beta ^\top (x - \hat \alpha) > \ln \frac{n_1}{n_0} \given x \sim \cN(\mu_1,\Sigma)\right)\nonumber\\
&= \frac{1}{2} \PP \left(\frac{\hat \beta ^ \top (x - \mu_0)}{\| \hat \beta \|_\Sigma} \leq  \frac{\hat \beta ^ \top (\hat \alpha - \mu _0) + \frac{n_1}{n_0}}{\|\hat \beta \|_\Sigma} \given x\sim \cN(\mu_0,\Sigma)\right) \nonumber \\
& \quad + \frac{1}{2} \PP \left(\frac{\hat \beta ^\top(x-\mu_1)}{\| \hat \beta \|_\Sigma} > \frac{\hat \beta ^\top(\hat \alpha - \mu_1)+\ln \frac{n_1}{n_0}}{\|\hat \beta\|_\Sigma}\given x \sim \cN(\mu_1,\Sigma)\right) \nonumber\\
&= \frac{1}{2} \Phi\left(\left[\betahat^\top  (\alphahat - \mu_0)+\ln{\frac{n_1}{n_0}}\right]/\| \betahat \|_\Sigma\right) + \frac{1}{2} \Phi\left(\left[\betahat^\top  (\mu_1 - \alphahat)+\ln{\frac{n_0}{n_1}}\right]/\|\betahat\|_\Sigma\right).\label{apeq:risk}
\end{align}
Therefore, it suffices to find the limits of 
\begin{align*}
q_0 = \left(\betahat^\top  (\alphahat - \mu_0)+\ln\frac{n_1}{n_0}\right)/\| \betahat \|_\Sigma \quad \textrm{and} \quad
q_1 = \left(\betahat^\top  (\mu_1 - \alphahat)+\ln\frac{n_0}{n_1}\right)/\|\betahat\|_\Sigma,
\end{align*}
since the Gaussian CDF $\Phi$ is continuous.

To further aid our analysis, we characterize the distributions of $\hat{\alpha}$, $\hat{\beta}$, and $\hat{\Sigma}$. Specifically, we make the following change of variables:
\begin{equation}\label{apeq:changeofvar}
\begin{split}
\hat\mu_0 &\eqd \frac{1}{\sqrt{n_0}}\Sigma^{\frac{1}{2}}z_0 + \mu_0,\\
\hat\mu_1 &\eqd\frac{1}{\sqrt{n_1}}\Sigma^{\frac{1}{2}}z_1 + \mu_1,\\
\hat{\Sigma}~ &\eqd \frac{1}{n-2}\Sigma^{\frac{1}{2}}Z^\top Z\Sigma^{\frac{1}{2}},
\end{split}
\end{equation}
where $z_0,z_1 \sim \cN(0,I_p)$ and $Z\in\RR^{(n-2)\times p}$ with each element $Z_{i,j}\sim \cN(0,1)$. Note that $z_0, z_1$ and $Z$ are independent with each other and $Z^\top Z$ is a Wishart matrix by Lemma \ref{lemma:independent}.

\noindent $\bullet$ {\bf Case 1. $0 < \gamma < 1$}. We present in detail how to characterize the asymptotic limit of $q_0$, and $q_1$ follows a similar argument. We tackle the numberator and denominator of $q_0$ separately. Using Lemma \ref{lemma:SLLN}, we check
\begin{align*}
\ln \frac{n_1}{n_0} \overset{\rm a.s.}{\longrightarrow} \ln \frac{\pi_1}{\pi_0} = \ln \frac{\gamma_0}{\gamma_1}.
\end{align*}
Therefore, we temporarily omit the threshold term $\ln (n_1/n_0)$ in both $q_0, q_1$ to ease the presentation. By the change of variables formula in \eqref{apeq:changeofvar} and some manipulation, we deduce
\begin{align}
\hat{\beta}^\top (\hat{\alpha} - \mu_0) & = - \frac{1}{2} (\mu_d + \frac{1}{\sqrt{n_0}}z_0 - \frac{1}{\sqrt{n_1}}z_1)^\top \left(\frac{Z^\top Z}{n-2}\right)^\dagger (\mu_d - \frac{1}{\sqrt{n_0}} z_0 - \frac{1}{\sqrt{n_1}} z_1)\nonumber\\
& =\frac{1}{2n_0} z_0^\top \left(\frac{Z^\top Z}{n-2}\right)^\dagger z_0  - \frac{1}{2}(\mu_d - z_1/\sqrt{n_1})^\top \left(\frac{Z^\top Z}{n-2}\right)^\dagger (\mu_d - z_1/\sqrt{n_1}) \label{apeq:q0numeratorstep1}\\
& \overset{d}{=} \underbrace{\left[ -\frac{1}{2}\|\mu_d -\tfrac{1}{\sqrt{n_1}}z_1\|_2^2 + \frac{1}{2n_0} \|z_0\|_2^2 \right]}_{A} \times \underbrace{\frac{1}{p} \tr\left(\frac{Z^\top Z}{n-2}\right)^{\dagger}}_{B}. \label{apeq:q0numeratorstep2}
\end{align}
where $\mu_d = \Sigma^{-1/2}(\mu_0 - \mu_1)$ and the last equality follows from the isotropicity of the Wishart matrix $Z^\top Z$ in lemma \ref{lemma3:isotropicity},
In the sequel, we establish the limits of terms $A$ and $B$, since they are independent.

\noindent {\bf Asymptotic convergence of $A$}. We show
\begin{align}\label{apeq:A}
A \overset{\rm a.s.}{\longrightarrow} -\frac{1}{2}(\Delta^2 + \gamma_1 - \gamma_0).
\end{align}
To see the result above, we expand term $A$ as
\begin{align*}
A=-\frac{1}{2}\norm{\mu_d}^2 + \frac{1}{\sqrt{n_1}}\mu_d^\top z_1 -\frac{1}{2n_1}\norm{z_1}^2 + \frac{1}{2n_0}\norm{z_0}^2.
\end{align*}
Recall that in Lemma \ref{lemma:SLLN}, we verify that $\frac{1}{n_1}\norm{z_1}_2^2 = \frac{p}{n_1} \frac{1}{p} \norm{z_1}_2^2 \overset{\rm a.s.}{\longrightarrow} \gamma_1$ and $\frac{1}{n_0}\norm{z_0}_2^2 \overset{\rm a.s.}{\longrightarrow} \gamma_0$ by the strong law of large numbers again. Using the Borel-Cantelli lemma, we claim $\frac{1}{\sqrt{n_1}} \mu_d^\top z_1 \overset{\rm a.s.}{\longrightarrow} 0$. Invoking the assertion in Theorem \ref{theorem1} that $\norm{\beta^*}_{\Sigma}^2 = \norm{\mu_d}_2^2 = \Delta^2$, we deduce the desired convergence of $A$.

\noindent {\bf Asymptotic convergence of $B$}. We next show
\begin{align}\label{apeq:B}
B \arrowas \frac{1}{1-\gamma}.
\end{align}
The convergence result above utilizes the Stieljes transformation of the Marchenko-Pastur law. Specifically, by the Bai-Yin theorem \citep{bai1993}, for $\gamma < 1$, 
\begin{equation*}
    \zeta_{\min}\left(\frac{Z^\top Z}{n-2}\right)\geq \frac{1}{2}(1-\sqrt{\gamma})^2,
\end{equation*}
which implies that $Z^\top Z$ is almost surely invertible.
Conditioned on $Z^\top Z$ being invertible, we rewrite $B$ as 
\begin{align*}
B = \frac{1}{p}\tr\left(\frac{Z^\top Z}{n-2}\right)^{-1} = \frac{1}{p} \sum_{i=1}^p \frac{1}{s_i} = \int \frac{1}{s} d F_{\frac{Z^\top Z}{n-2}} (s),
\end{align*}
where $s_i$'s denote the eigenvalues of $Z^\top Z/(n-2)$ and $F_{M}(a) = \frac{1}{p} \sum_{i=1}^p \mathbbm{1}\{\lambda_i(M) \leq a\}$ is the empirical measure of eigenvalues of $M$.

Now apply the Marchenko-Pastur theorem \cite{marchenko-pastur}, which says that $F_{Z^\top Z/(n-2)}$ converges weakly, almost surely, to the Marchenko-Pastur law $F_\gamma$ (depending only on $\gamma$). Invoking the Portmanteau theorem \citep{marchenko-pastur}, weak convergence is equivalent to the convergence in expectation of all bounded functions $h$, that are continuous except on a set of zero probability under the limiting measure. Defining $h(s)=1/s \cdot \mathds{1}\{s\geq a/2\}$, where we abbreviate $a=(1-\sqrt{\gamma})^2$, it follows that as $n,p \to \infty$, almost surely,
    \begin{equation*}\begin{split}
        \int_{a/2}^\infty \frac{1}{s} dF_{\frac{Z^\top Z}{n-2}}(s) \mathop{\to} & \int_{a/2}^\infty \frac{1}{s} dF_\gamma(s)
    \end{split}\label{apeq.dF}
    \end{equation*}
We can remove the lower limit of integration on both sides above; for the right-hand side, this follows since support of the Marchenko-Pastur law $F_\gamma$ is $[a,b]$, where $b = (1+\sqrt{\gamma})$; for the left-hand side, this follows again by the Bai-Yin theorem \cite{bai1993} (which as already stated, implies the smallest eigenvalues of $Z^\top Z/(n-2)$) is almost surely greater than $a/2$ for large enough n). Thus the last display implies that as $n,p\to \infty$, almost surely,
\begin{equation*}
B \to \int \frac{1}{s} F_\gamma(s)
\end{equation*}

Thanks to the Stieljes transformation of the Marchenko-Pastur law, we can explicitly compute $\int \frac{1}{s} d F_\gamma(s)$. In particular, the Stieljes transformation of $F_\gamma$ is defined as
\begin{align*} 
m(\zeta) = \int \frac{1}{s - \zeta} d F_\gamma(s) \quad \textrm{with} \quad \zeta \in \CC.
\end{align*}
Taking $\zeta \rightarrow 0$, we obtain $m(0) = \lim_{\zeta \rightarrow 0} m(\zeta) = \frac{1}{1-\gamma}$ by Lemma \ref{lem.MPLaw}. Consequently, we establish $B \arrowas \frac{1}{1-\gamma}$.

Substituting \eqref{apeq:A} and \eqref{apeq:B} into \eqref{apeq:q0numeratorstep2} yields
\begin{align}\label{apeq:numerator}
\hat{\beta}^\top (\hat{\alpha} - \mu_0) \arrowas -\frac{1}{2}(\Delta^2 + \gamma_1 - \gamma_0) \frac{1}{1-\gamma}.
\end{align}

Next we consider the denominator $\|\hat{\beta}\|_{\Sigma}$ in $q_0$. Applying the change of variables in \eqref{apeq:changeofvar} and using Lemma \ref{lemma3:isotropicity}, analogous to \eqref{apeq:q0numeratorstep2}, we derive
\begin{align}
\|\widehat\beta \|_\Sigma^2
= & (\widehat{\mu}_0-\widehat{\mu}_1)^\top \widehat\Sigma^\dagger \Sigma \widehat\Sigma^\dagger(\widehat{\mu}_0-\widehat{\mu}_1)\nonumber  \\
= & ({\mu}_d + \frac{1}{\sqrt{n_0}}{z}_0 - \frac{1}{\sqrt{n_1}}{z}_1)^\top  \left(\left(\frac{Z^\top Z}{n-2}\right)^{\dagger}\right)^2 ({\mu}_d + \frac{1}{\sqrt{n_0}}{z}_0 - \frac{1}{\sqrt{n_1}}{z}_1)\nonumber \\
\overset{d}{=}& \underbrace{\norm{\mu_d + \frac{1}{\sqrt{n_0}}z_0 - \frac{1}{\sqrt{n_1}}z_1}_2^2}_{A'} \times \underbrace{\frac{1}{p} \tr\left(\left(\frac{Z^\top Z}{n-2}\right)^\dagger\right)^2}_{B'}. \label{apeq:q0denominator}
\end{align}
The convergence of term $A'$ follows the same argument of term $A$ in the numerator, and we have
\begin{align}\label{apeq:A'}
A' \arrowas \Delta^2 + \gamma_1 + \gamma_0.
\end{align}
Conditioned on $Z^\top Z$ being invertible, term $B'$ can be written as
\begin{align*}
B' = \int \frac{1}{s^2} d F_{\frac{Z^\top Z}{n-2}}(s) \longrightarrow \int \frac{1}{s^2} dF_\gamma (s).
\end{align*}
To compute the limiting integral above, we differentiate the Stieljes transformation $m(\zeta)$. By sending $\zeta \rightarrow 0$ again, we can derive
\begin{align}\label{apeq:B'}
B' \arrowas \lim_{\zeta \rightarrow 0} \frac{d}{d\zeta} m(\zeta) = \frac{1}{(1-\zeta)^3}.
\end{align}
Substituting \eqref{apeq:A'} and \eqref{apeq:B'} into \eqref{apeq:q0denominator} yields
\begin{align}\label{apeq:divider}
\|\hat{\beta}\|_{\Sigma} \arrowas \sqrt{\Delta^2 + \gamma_1 + \gamma_0} \frac{1}{(1-\gamma)^{3/2}}.
\end{align} 

Combining \eqref{apeq:numerator} and \eqref{apeq:divider}, as well as putting the threshold term $\ln (n_1/n_0)$ back, we obtain
\begin{align}\label{apeq:q0}
q_0 \arrowas \frac{-\frac{\Delta^2 + \gamma_1 - \gamma_0}{2(1-\gamma)}+\ln \frac{\gamma_0}{\gamma_1}}{\sqrt{\Delta^2 + \gamma_1 + \gamma_0} \frac{1}{(1-\gamma)^{3/2}}}.
\end{align}
The same argument of analyzing $q_0$ applies to $q_1$, and therefore, we have
\begin{align}\label{apeq:q1}
q_1 \arrowas \frac{-\frac{\Delta^2 + \gamma_0 - \gamma_1}{2(1-\gamma)}+\ln \frac{\gamma_1}{\gamma_0}}{\sqrt{\Delta^2 + \gamma_1 + \gamma_0} \frac{1}{(1-\gamma)^{3/2}}}.
\end{align}
To complete the proof in the case of $0 < \gamma < 1$, we plugging \eqref{apeq:q0} and \eqref{apeq:q1} into \eqref{apeq:risk}.

\noindent $\bullet$ {\bf Case 2. $\gamma > 1$}. The goal is still to find the limits of $q_0$ and $q_1$. Consider $q_0$ first. We observe that both \eqref{apeq:q0numeratorstep2} and \eqref{apeq:q0denominator} are valid for $\gamma > 1$. However, a key difference is that $Z^\top Z$ is rank deficient in the limit considering $\gamma > 1$. To resolve this issue, we observe that $Z^\top Z$ and $ZZ^\top$ share all nonzero eigenvalues. Therefore, we can replace $Z^\top Z$ by $ZZ^\top$ in \eqref{apeq:q0numeratorstep2} and \eqref{apeq:q0denominator} without changing their values. Such a reasoning is justified by Lemma \ref{lemma:trace}.
We can now rewrite \eqref{apeq:q0numeratorstep2} as
\begin{align}
\hat{\beta}^\top (\hat{\alpha} - \mu_0) &\overset{d}{=} \left[ -\frac{1}{2}\|\mu_d -\tfrac{1}{\sqrt{n_1}}z_1\|_2^2 + \frac{1}{2n_0} \|z_0\|_2^2 \right] \times \frac{(n-2)^2}{p^2}\frac{1}{n-2}\tr\left(\frac{Z Z^\top}{p}\right)^{\dagger}. \nonumber\\
&\arrowas -\frac{1}{2}\left[\Delta^2 + \gamma_1 - \gamma_0 \right] \times \frac{1}{\gamma(\gamma - 1)}.\label{apeq:numeratorresult}
\end{align}
Similarly we can rewrite \eqref{apeq:q0denominator} as 
\begin{align}
    \| \hat \beta \|_\Sigma^2 & \eqd \norm{ \mu_d + \frac{1}{\sqrt{n_0}} z_0 - \frac{1}{\sqrt{n_1}} z_1 }_2^2 \times \frac{(n-2)^3}{p^3} \frac{1}{n-2} \tr \left( \left(\frac{Z^\top Z}{n-2}\right)^\dagger \right)^2 \nonumber \\
    \label{apeq:q0denominatorresult}&\arrowas (\Delta^2 + \gamma_0 + \gamma_1) \times \frac{1}{(\gamma-1)^3}. \\
\end{align}
Combining \eqref{apeq:numeratorresult} and \eqref{apeq:q0denominatorresult}, as well as putting the threshold term $\ln(n_1/n_0)$ back, we obtain
\begin{align}
q_0 & \arrowas \frac{-\frac{\Delta^2 + \gamma_1 - \gamma_0}{2\gamma(\gamma-1)} + \ln \frac{\gamma_0}{\gamma_1}}{\sqrt{\Delta^2 + \gamma_1 + \gamma_0} \frac{1}{(\gamma-1)^{3/2}}}, \nonumber
\end{align}
The same argument of analyzing $q_0$ applies to $q_1$, and therefore, we have
\begin{align}
q_1 & \arrowas \frac{-\frac{\Delta^2 + \gamma_0 - \gamma_1}{2\gamma(\gamma-1)} + \ln \frac{\gamma_1}{\gamma_0}}{\sqrt{\Delta^2 + \gamma_1 + \gamma_0} \frac{1}{(\gamma-1)^{3/2}}}. \nonumber
\end{align}
The misclassification error in the case of $\gamma > 1$ follows by substituting $q_0, q_1$ above into \eqref{apeq:risk}. The proof is complete.
\end{proof}

\begin{figure}[!tp]
    \begin{center}
        \includegraphics[width=0.5\textwidth]{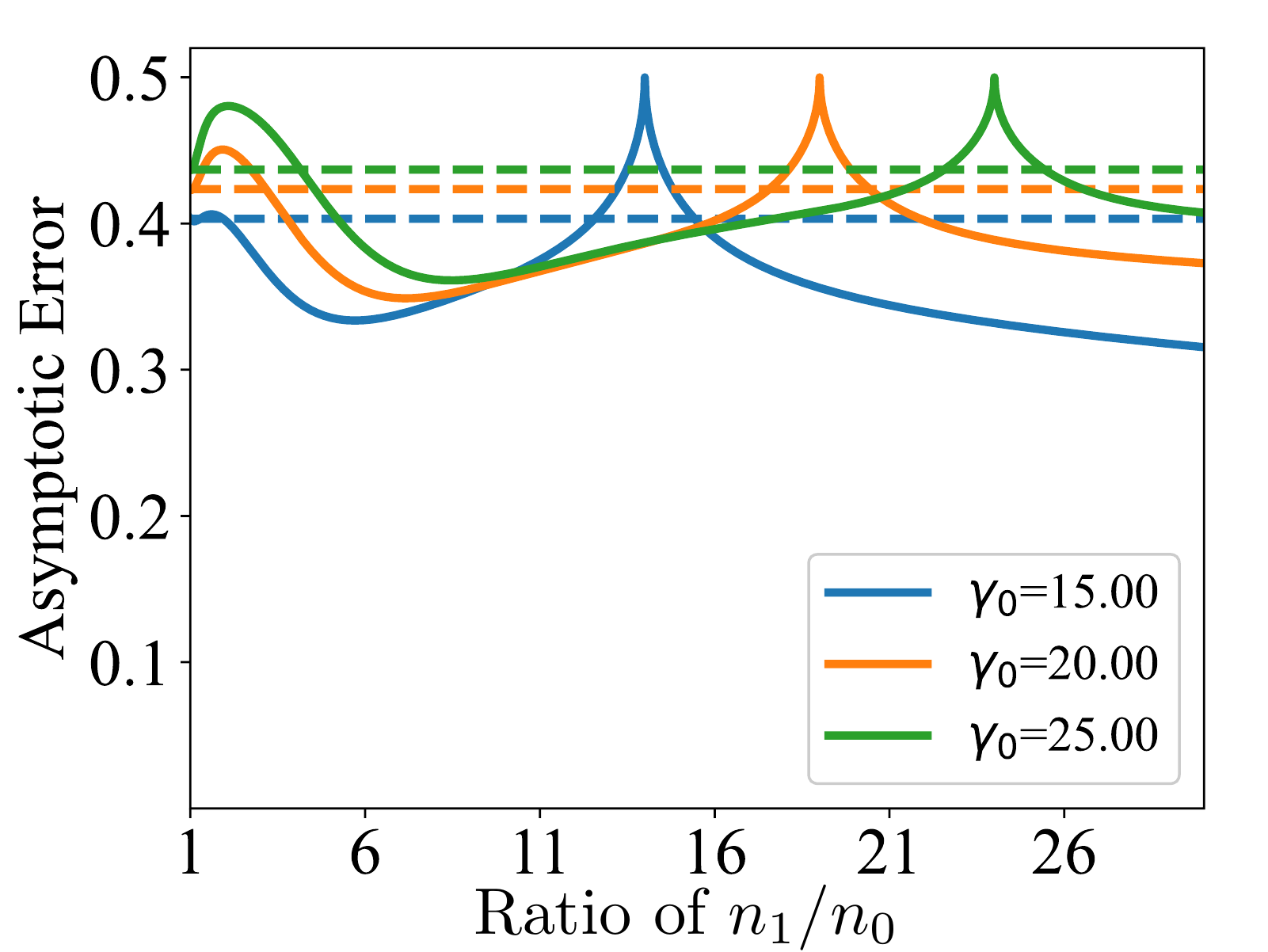}
    \end{center}
    \vspace{-0.2cm}
    \caption{Misclassification error in the highly overparametrized regime as a function of $n_1/n_0$, with $\Delta^2 = 9$.}
    \label{fig:highlyoverparametried}
    \vspace{-0.5cm}
\end{figure}

\save{\section{Conclusion and Discussion}
\label{sec:ConclusionAndDiscussion}}

This paper provides a theoretical analysis on the performance of LDA under label shift, in both the under- and over-parametrized regime. We explicitly quantify the misclassification error in the proportional limit of $n \rightarrow \infty$ and $p/n_{\ell}  \rightarrow \gamma_{\ell}$ for $\ell=0,1$, where $\gamma_{\ell} > 0$ is a constant.
Our theory shows a peaking phenomenon when the sample size is close to the data dimension.
We demonstrate a {\it phase transition} phenomenon about data imbalance: The misclassification error exhibits different behaviors as the two-class ratio $n_1/n_0$ varies, depending on the value of $\gamma_0$. We clearly characterize the three behaviors of the misclassification error in the underparametrized, lightly overparametrized, and overparametrized regions depending on $\gamma_0$. We also investigate the regularized LDA, and show that the peaking and phase transition phenomenons disappear when the regularization becomes strong.

Additionally, in the highly overparametrized regime, the misclassification error under label shift has multiple maxima, as shown in Figure \ref{fig:highlyoverparametried}. We believe that the multi-peaking behavior of the misclassification error reflects the intertwined interaction between the lack of information in data and the model mismatch under label shift.

\appendix

\section{Proof of Phase Transition in Section \ref{sec:phaselda}}\label{appendix:pfphasetransition}

\paragraph{Misclassification error as $n_1/ n_0 \rightarrow \infty$}

In this section, we will give a proof to show the misclassificaiton error tends to 0.5 when $n_1/n_0\to\infty$. When the training data set is extremely imbalanced with $n_1/n_0\rightarrow \infty$, i.e., $\gamma_0/\gamma_1\to\infty$, the Bayes classifier tends to classify all data points to class 1. This leads to the following limits,
\begin{equation*}
    \frac{\Delta^2+\gamma_0-\gamma_1}{\sqrt{\Delta^2+\gamma_0+\gamma_1}}\to +\infty,\qquad \frac{\Delta^2+\gamma_1-\gamma_0}{\sqrt{\Delta^2+\gamma_0+\gamma_1}}\to -\infty.
\end{equation*}
Then the limit of misclassification error is given by
\begin{equation*}
    \cR(f_{\widehat \alpha, \widehat \beta}^{\widehat b})\overset{\rm a.s.}\longrightarrow\frac{1}{2}\Phi(-\infty)+\frac{1}{2}\Phi(+\infty)=0.5.
\end{equation*}
The proof is complete. 
\paragraph{Phase transition knots}

In this section, we provide theoretical justifications of the phase transition knots given in section \ref{sec:phaselda}.
Denote the asymptotic misclassification error in \eqref{eq:theorem-underparametrized} and \eqref{eq:theorem-overparametrized} as $\cR(\gamma_0,\gamma_1)$. 
We fix $\gamma_0$ and let $\gamma_1$ vary starting from the balanced case with $\gamma_1=\gamma_0$. 

The transition knots are obtained by a local analysis about the instantaneous change of misclassification error as $n_1/n_0$ slightly increases from $1$. We observe that, as $n_1/n_0$ slightly increases from $1$,  the misclassification error decreases in Phase I and III, and increases in Phase II. Notice that $\gamma_1$ slightly decreases from $\gamma_0$ as $n_1/n_0$ slightly increases from $1$. 

The instantaneous change of $\cR(\gamma_0,\gamma_1)$ with respect to $\gamma_1$ can be characterized by 
the following partial derivative,  $\frac{\partial}{\partial\gamma_1}\cR(\gamma_0,\gamma_1)\given_{\gamma_1=\gamma_0}$:

\begin{align}
    \hspace{-3cm}\frac{\partial}{\partial\gamma_1}\cR\given_{\gamma_1=\gamma_0} = \frac{\Delta^2 \phi\left(\frac{-\Delta^2(1-\frac{1}{2}\gamma_0)^{\frac{1}{2}}}{2\sqrt{\Delta^2+2\gamma_0}}\right)}{16(\Delta^2+2\gamma_0)^{\frac{3}{2}}(1-\frac{1}{2}\gamma_0)^{\frac{1}{2}}}
    \times
    [4+\Delta^2], \nonumber
\end{align}
for $\gamma_0<2$, and
\begin{align}
    \frac{\partial}{\partial\gamma_1}\cR\given_{\gamma_1=\gamma_0} = \frac{\Delta^2\phi\left(\frac{-\Delta^2(\frac{1}{2}\gamma_0-1)^{\frac{1}{2}}}{\gamma_0\sqrt{\Delta^2+2\gamma_0}}\right)}{4\gamma_0^2(\Delta^2+2\gamma_0)^{\frac{3}{2}}(\frac{1}{2}\gamma_0-1)^{\frac{1}{2}}}
    \times
    \underbrace{[4\gamma_0^2-(12-\Delta^2)\gamma_0-4\Delta^2]}_{Q(\gamma_0,\Delta)}, \nonumber
\end{align}
for $\gamma_0>2$, where $\phi$ is the probability density function of the standard normal distribution. The  $Q(\gamma_0,\Delta)$ term is a quadratic function with two roots of opposite signs. The positive root is $\gamma_b=\frac{1}{8}(12-\Delta^2+\sqrt{\Delta^4+40\Delta^2+144})$. The sign of the above partial derivative has the following cases:

\begin{itemize}

\item When $\gamma_0\in (0,2)$, $\frac{\partial}{\partial\gamma_1}\cR(\gamma_0,\gamma_1)\given_{\gamma_1=\gamma_0}$ is always positive. As a result, $\cR(\gamma_0,\gamma_1)$ decreases as $\gamma_1$ decreases from $\gamma_0$, which corresponds to Phase I.

\item When $\gamma_0 \in (2,\gamma_b)$, $\frac{\partial}{\partial\gamma_1}\cR(\gamma_0,\gamma_1)\given_{\gamma_1=\gamma_0}$ is negative. In this case $\cR(\gamma_0,\gamma_1)$ increases as $\gamma_1$ decreases from $\gamma_0$, which corresponds to Phase II.

\item When $\gamma_0 \in (\gamma_b,+\infty)$, $\frac{\partial}{\partial\gamma_1}\cR(\gamma_0,\gamma_1)\given_{\gamma_1=\gamma_0}$ is positive. In this case $\cR(\gamma_0,\gamma_1)$ decreases as $\gamma_1$ decreases from $\gamma_0$, which includes Phase III.

\end{itemize}

\section{Proofs of Lemmas in Section \ref{subsec:lemma}}
\label{subsec:lemmaproof}

\subsection{Proofs of Lemma \ref{lemma:independent} and Lemma \ref{lemma3:isotropicity} }\label{appendix:pflma}

\begin{proof}[Proof of lemma \ref{lemma:independent}]
    We will prove this result by Basu theorem, i.e., we will show that ($\widehat\mu-\mu$) is a complete and sufficient statistic and $\widehat \Sigma$ is an auxiliary statistic w.r.t. $\mu$. 

    First, we show $\widehat\mu$ is a complete statistic. We need to check that for any $\mu$ and measurable function $g$, $\EE [g(\hat{\mu})] = 0$ for any $\mu$ implies $\PP (g(\hat{\mu}) = 0) = 1$ for any $\mu$. Indeed, for any measurable function $g$ such that the expectation of $g(\widehat {\mu})$ over sample space $(x_1,x_2,\dots,x_n)$ is zero, i.e.,
    \begin{equation}\label{apeq:independenceasm}
        \EE [g(\widehat {\mu})] = 0\quad \textrm{for any}~\mu,
    \end{equation}
    we can derive $\PP (g(\hat{\mu} - \mu) = 0) = 1$ by taking derivatives of Equation \eqref{apeq:independenceasm} w.r.t. $\mu$ recursively,
    \begin{equation*}
        \EE \bigg[ h(\widehat {\mu}) g(\widehat {\mu})\bigg] = 0,  \text{for any polynomial $h$},
    \end{equation*}
    and therefore $\widehat {\mu}$ is a complete statistic w.r.t. parameter ${\mu}$.

    To prove $\widehat\mu$ is also a sufficient statistic for ${\mu}$, we need to show that given the statistic $\widehat \mu$ the conditional distribution of $x_1,...,x_n$ does not depend on $\mu$. Note that $\widehat \mu$ has a multivariate normal distribution, i.e., $\widehat {\mu} \sim \cN({\mu}, \frac{1}{n}\Sigma)$, since $\widehat {\mu} = \frac{1}{n} \sum_{i=1}^n x_i$ is a linear combination of i.i.d. multivariate normal vectors $x_1,x_2, \dots, x_n$. The pdf of $\hat\mu$ and the joint distribution of $x_1, x_2, \dots, x_n$ are given by
    \begin{align}
        \label{apeq:muhat}f(\widehat {\mu} ) &= \frac{1}{(2\pi)^{\frac{p}{2}}\vert\frac{1}{n}\Sigma\vert^{\frac{1}{2}}} \exp \left(-\frac{n}{2}(\widehat {\mu} - {\mu})^\top  \Sigma^{-1} (\widehat {\mu} - {\mu})\right),\\
        f(x_1,\dots,x_n) &= \frac{1}{(2\pi)^\frac{np}{2} \vert\Sigma\vert^\frac{n}{2}} \exp\left(-\sum_{i=1}^n \frac{1}{2}(x_i-{\mu})^\top  \Sigma^{-1} (x_i-{\mu})\right).\nonumber
    \end{align}
    
    \par The joint density function of $x_1, \dots, x_n$ and $\hat \mu$ is given by
    \begin{equation}\label{apeq:xi-muhat}
        f(x_1,\dots,x_n,\muhat) =f(x_1,\dots,x_n) \mathds{1}\left(\muhat=\frac{1}{n}(x_1+x_2+\dots+x_n)\right).
    \end{equation}
    
    \par By taking the fraction of \eqref{apeq:muhat} and \eqref{apeq:xi-muhat}, the conditional density of $x_1,\dots, x_n$ given $\hat\mu$ is
    \begin{equation}
        f\left(x_1,\dots,x_n\given\widehat {\mu}\right)= C  \exp\left(-\frac{1}{2}(x-\widehat {\mu})^\top  \Sigma^{-1} (x-\widehat {\mu})\right),
    \end{equation}
    where $C$ is a constant. By Fisher-Neyman factorization theorem \citep{lehmann2011fisher},  given the statistic $\widehat \mu$ the conditional distribution of $x_1,...,x_n$ does not depend on $\mu$ and therefore $\hat\mu$ is a sufficient statistic for ${\mu}$.
    
    Sample covariance has a distribution which doesn't depend on the parameter ${\mu}$.
    \begin{equation}
        \widehat \Sigma = \sum_{i=1}^n (x_i - \widehat {\mu})^\top  (x_i - \widehat {\mu}) = \sum_{i=1}^n z_i^\top z_i,
    \end{equation}
    and therefore it is a auxiliary statistic.

    Combining $\widehat \mu$ being a complete and sufficient statistic and $\widehat \Sigma$ being an auxiliary statistic, we obtain that $\widehat {\mu}$ and  $\widehat \Sigma$ are independent, by Basu Theorem.
\end{proof}

\begin{proof}[Proof of lemma \ref{lemma3:isotropicity}]
    The following isotropic property of Wishart distribution has been given by \citet{wang2018}. For any orthogonal matrix $U \in \RR^{p \times p}$, we have
    \begin{equation}
        U^\top \left(\frac{Z^\top Z}{n-2}\right) U \sim \mathcal{W}(I_p,n-2).
    \end{equation}
    We next apply this property to the left-hand side of equation \eqref{apeq:wishart},
    \begin{equation}\begin{split}
        z^\top \left(\frac{Z^\top Z}{n-2}\right)^{\dagger}z &=z^\top U_i U_i^\top \left(\frac{Z^\top Z}{n-2}\right)^\dagger  U_i U_i^\top z\\
                 &=\norm z^2 e_i^\top \left( U_i^\top \left(\frac{Z^\top Z}{n-2}\right)U_i\right)^{\dagger}  e_i\\
                 &\eqd \norm z^2 e_i^\top \left(\frac{Z^\top Z}{n-2}\right)^{\dagger} e_i,
    \end{split}\end{equation}
    where $U_i$ is a orthogonal matrix that transforms the vector $z$ to canonical basis vector $e_i$, i.e., 
    \begin{equation}
        U_i^\top z = \norm z_2  e_i, \quad i=1,2,\dots,p.
    \end{equation}
    We can further simplify the product $ z ^\top (\frac{Z^\top Z}{n-2})^{\dagger}  z$ by taking an average over the index $i$,
    \begin{equation}\begin{split}
        z^\top \left(\frac{Z^\top Z}{n-2}\right)^\dagger z &\mathop{=}\limits^d \frac{1}{p} \| z \|_2^2 \sum_{i=1,...,p} e_i^\top \left(\frac{Z^\top Z}{n-2}\right)^{\dagger} e_i\\
        & = \frac{1}{p}\|{z}\|^2_2 \tr\left(\left(\frac{Z^\top Z}{n-2}\right)^\dagger\right),
    \end{split}\end{equation}
    where we get the isotropicity of Wishart distribution in Equation \eqref{apeq:wishart}.
\end{proof}

\subsection{Proofs of Lemma \ref{lemma:trace}, Lemma \ref{lemma:SLLN} and Lemma \ref{lem.MPLaw}}\label{subsec:lemmaproof2}

\begin{proof}[Proof of Lemma \ref{lemma:trace}]
    We use the eigenvalue decomposition of $Z^\top Z=U^\top D U$ to simplify the left-hand side of Equation \eqref{apeq:trace},
    \begin{equation*}\begin{split}
        \mathrm{tr}[(Z^\top Z)^\dagger] &=\tr\left(U D^\dagger U^\top\right) \\
        &=\tr\left(D^\dagger\right)\\
        &=\sum\limits_{s \in \lambda(Z^\top Z),s\neq 0} \frac{1}{s}.
    \end{split}\end{equation*}
    \par The result above implies that the trace of the pseudo-inverse of $Z^\top Z$ is equal to the sum of the reciprocal of its eigenvalues. By the same arguments on $ZZ^\top$, we can show that 
    \begin{equation*}
         \mathrm{tr}[(ZZ^\top )^\dagger] =\sum\limits_{s \in \lambda(ZZ^\top),s\neq 0} \frac{1}{s}.
    \end{equation*}
    Then we deduce the desired result by the fact that the set of non-zero eigenvalues of $ZZ^\top $ matches that of $Z^\top Z$.
\end{proof}

\begin{proof}[Proof of Lemma \ref{lemma:SLLN}]
    \par We first compute the limit of $\frac{1}{\sqrt{p}}\mu_d^\top z$. The linear combination of multivariate normal random vector $z\sim \mathcal{N}(0,I_p)$ is a normal random variable, namely, $\frac{1}{\sqrt{p}}\mu_d^\top z \sim \mathcal{N}(0, \frac{1}{p}{\mu}_d^\top \mu_d)$. From the concentration inequality of the normal random variable \cite{IntroProbTheoryApplication}, we have
    \begin{equation}\label{apeq:concentration}
        \mathbb{P}\left(\left\vert\frac{1}{\sqrt{p}}\mu_d^\top z\right\vert\geq \frac{\epsilon}{\sqrt{p}}\|\mu_d\|_2\right)\leq 2 \frac{1}{\sqrt{2\pi}\epsilon} e^{-\epsilon^2/2} \quad \text{for all }x\geq 0.
    \end{equation}
    Combining Equation \eqref{apeq:concentration}, $e^{-x}<\frac{1}{x}$ for $x>0$ and $\|\mu_d\|_2 < 2\Delta$ for a sufficiently large $p$, the sum of the probabilities of $\frac{1}{\sqrt{p}}\vert\mu_d^\top {z}\vert > \epsilon $ is finite, for any positive $\epsilon > 0$, i.e.,
    \begin{equation}\begin{split}
    \sum_{p=1}^{\infty}\mathbb{P}\left(\frac{1}{\sqrt{p}}\vert\mu_d^\top {z}\vert > \epsilon\right) &\leq \sum_{p=1}^{\infty} \frac{\|\mu_d\|_2}{\sqrt{2\pi p}\epsilon} e^{-(\epsilon^2p)/(2\|\mu_d\|_2^2)}\\
    &<\sum_{p=1}^{\infty}\frac{2\|\mu_d\|_2^3}{\sqrt{2\pi}\epsilon^3}~\frac{1}{p^{3/2}}\\
    &<\infty. \nonumber
    \end{split}\end{equation}
    
    By the Borel-Cantelli lemma, we have 
    $$\frac{1}{\sqrt{p}}\mu_d^\top {z} \overset{\rm a.s.} \longrightarrow 0.$$
    
    \par We next consider the limit of $\frac{1}{p}z_\ell^\top z_\ell$. Since $z_\ell$ satisfies the chi-squared distribution independently with expectation $\mathbb{E}[z_{\ell,i}^2]=1$ and finite variance $\mathbb{V}ar (z_{\ell,i}^2)=2$, we know the average of the squared elements in $z_\ell$ converges to the expectation almost surely by the strong law of the large numbers, namely, 
    $$\frac{1}{p}z_\ell^\top z_\ell \arrowas 1. $$ 
    \par By the same arguments given above, $n_\ell$ satisfies the binomial distribution $B(n_9+n_1,\pi_\ell)$ which is composed of $n_0+n_1$ independent Bernoulli distribution with expectation $\pi_0$ and Variance $\pi_0\pi_1$. From the strong law of large numbers, we have 
    $$\frac{n_\ell}{n_0+n_1}\arrowas \pi_\ell.$$
\end{proof}

\begin{proof}[Proof of Lemma \ref{lem.MPLaw}]
We first derive the expression of $m(\zeta)$. The Marchenko-Pastur law is supported on a compact subset of $\cR^+$, i.e., $\mathrm{supp}(F_{\gamma})\subset [a,b]$ where
\begin{equation}
    a=(1-\sqrt{\gamma})^2, \quad \text{and} \quad b=(1+\sqrt{\gamma})^2. \nonumber
\end{equation}

     Let $\{z_k\}$ be a sequence of complex numbers such that $\mathrm{Im}(z_k)>0, \mathrm{Re}(z_k)=\zeta$ for any $k$ and $\lim_{k\rightarrow \infty}z_k=\zeta$.
    Consider the sequence of integral
    $$
    \int_{a}^b \frac{1}{s-z_k} dF_\gamma(s).
    $$
    
For any $k$, $0<\gamma<1$ and $s>a$, we have $$\left\vert\frac{1}{s-z_k}\right\vert\leq\left\vert\frac{1}{s}\right\vert<\frac{1}{(1-\sqrt{\gamma})^2}<\infty.$$
    By the dominated convergence theorem, we have
    \begin{align}
        \int \frac{1}{s-\zeta} dF_\gamma(s)&=\int_{a}^b \lim_{k\rightarrow\infty}\frac{1}{s-z_k} dF_\gamma(s)=\lim_{k\rightarrow\infty} \int_{a}^b \frac{1}{s-z_k} dF_\gamma(s)\nonumber\\
        &=\lim_{k\rightarrow\infty} \int \frac{1}{s-z_k} dF_\gamma(s). \quad 
        \label{eq.intlim}
    \end{align}
    To compute $\int \frac{1}{s-z_k} dF_\gamma(s)$,    \citet[Lemma 3.11]{Bai-and-Silverstein2010} gives
    \begin{align}
        \int \frac{1}{s-z_k} dF_\gamma(s)= \frac{1-\gamma-z_k+\sqrt{(z_k-\gamma-1)^2-4\gamma}}{2\gamma z_k}.
        \label{eq.Bai}
    \end{align}
    
    According to the definition of the square root of complex numbers in \citet[Equation (2.3.2)]{Bai-and-Silverstein2010}, the real part of $\sqrt{(z_k-\gamma-1)^2-4\gamma}$ has the same sign as that of $z_k-\gamma-1$. Since $\mathrm{Re}(z_k)=\zeta\leq0, \gamma>0$, the real part of $\sqrt{(z_k-\gamma-1)^2-4\gamma}$ is negative and gives
    \begin{align}
        \lim_{k\rightarrow \infty}  \sqrt{(z_k-\gamma-1)^2-4\gamma}= -\sqrt{(\zeta-\gamma-1)^2-4\gamma}.
        \label{eq.squarelimit}
    \end{align}
    Substituting (\ref{eq.squarelimit}) and (\ref{eq.Bai}) into (\ref{eq.intlim}) gives rise to (\ref{eq.MP.m}).
    
    We then compute $m(0)$. When substituting $\zeta=0$ into (\ref{eq.MP.m}), both the numerator and the denominator are 0. Here we apply L'Hospital's rule:
    \begin{align*}
        m(0)
        &=\lim_{\zeta\rightarrow 0}\frac{1-\gamma-\zeta-\sqrt{(\zeta-\gamma-1)^2-4\gamma}}{2\gamma \zeta} \\
        &=\lim_{\zeta\rightarrow 0} \frac{1}{2\gamma} \left(-1-\frac{\zeta-\gamma-1}{\sqrt{(\zeta-\gamma-1)^2-4\gamma}}\right)\\
        &=\frac{1}{2\gamma} \left(-1-\frac{-\gamma-1}{\sqrt{(-\gamma-1)^2-4\gamma}}\right) \nonumber\\
        &=\frac{1}{2\gamma} \left(-1+\frac{1+\gamma}{1-\gamma}\right) \nonumber\\
        &=\frac{1}{1-\gamma}.
        \label{eq.limint}
    \end{align*}

We next derive the expression of $\frac{d}{d\zeta}m(\zeta)$. To derive the expression, we first show that 
    \begin{align}
        \int \frac{1}{(s-\zeta)^2} dF_\gamma(s)= \lim_{z\rightarrow \zeta}\frac{d }{d z} \int \frac{1}{s-z} dF_\gamma(s) \quad \textrm{  for $z\in \CC$ with $\mathrm{Re}(z)=\zeta, \mathrm{Im}(z)>0$.}
    \end{align}
    Let $\{h_k\}$ be a set of complex numbers such that $\vert\mathrm{Re}(h_k)\vert\leq \vert\zeta\vert/2$ for any $k$ and $\lim_{k\rightarrow\infty}h_k=0$. For any $k$ and $s\geq a$, we have
    \begin{align*}
        \left\vert\frac{1}{(s-z-h_k)(s-z)}\right\vert\leq \frac{1}{(1-\sqrt{\gamma})^4}<\infty.
    \end{align*}
    By the dominated convergence theorem, we have
    \begin{align}
    \frac{\partial }{\partial z} \int \frac{1}{s-z} dF_\gamma(s)=
        &\lim_{k\rightarrow\infty}\frac{1}{h_k}\left[\int \frac{1}{s-(z+h_k)} dF_\gamma(s)-\int \frac{1}{s-z} dF_\gamma(s)\right] \nonumber\\
        =&\lim_{k\rightarrow\infty}\frac{1}{h_k}\int_{a}^{b} \frac{1}{s-z-h_k}-\frac{1}{s-z} dF_\gamma(s) \nonumber\\
        =&\lim_{k\rightarrow\infty}\int_{a}^{b} \frac{1}{(s-z-h_k)(s-z)} dF_\gamma(s)\nonumber \\
        =&\int_{a}^{b}\frac{1}{(s-z)^2} dF_\gamma(s),\nonumber
    \end{align}
    where the second equality holds since $F_{\gamma}(s)$ is supported on $[a,b]$.
    
    Since  
    \begin{align*}
        \left\vert\frac{1}{(s-z)^2}\right\vert\leq \frac{1}{(1-\sqrt{\gamma})^4}<\infty,
    \end{align*}
    for any $s\geq a$, we have
    \begin{align}
        \int \frac{1}{(s-\zeta)^2}dF_{\gamma}(s)&= \int_{a}^{b} \lim_{z\rightarrow \zeta}\frac{1}{(s-z)^2}dF_{\gamma}(s) =\lim_{z\rightarrow\zeta} \int_{a}^{b} \frac{1}{(s-z)^2}dF_{\gamma}(s)\nonumber\\
        &= \lim_{z\rightarrow\lambda}\frac{\partial }{\partial z} \int \frac{1}{s-z} dF_\gamma(s). \nonumber
    \end{align}
    Using (\ref{eq.Bai}), we have
    \begin{align}
        &\frac{\partial }{\partial z} \int \frac{1}{s-z} dF_\gamma(s) \nonumber\\&= \frac{(2\gamma z)\left[-1+\frac{z-\gamma-1}{\sqrt{(z-\gamma-1)^2-4\gamma}}\right]-(2\gamma)\left(1-\gamma-z+\sqrt{(z-\gamma-1)^2-4\gamma}\right)}{4\gamma^2z^2} \nonumber\\
        &=\frac{(2\gamma)(\gamma-1)+\frac{(2\gamma z)(z-\gamma-1)}{\sqrt{(z-\gamma-1)^2-4\gamma}}-2\gamma \sqrt{(z-\gamma-1)^2-4\gamma}}{4\gamma^2z^2}.
        \label{eq.partialBai}
    \end{align}
    Letting $z\rightarrow\zeta$ in (\ref{eq.partialBai}) and recall that the real part of $\sqrt{(z-\gamma-1)^2-4\gamma}$ is negative, one gets (\ref{eq.MP.dm}).
    
    To compute $\frac{d}{d\zeta}m(0)$, by L'Hopital's rule, we deduce
    \begin{align}
       \frac{d}{d\zeta}m(0)= \lim_{\zeta\rightarrow0} \frac{(2\gamma)(\gamma-1)-\frac{(2\gamma z)(\zeta-\gamma-1)}{\sqrt{(\zeta-\gamma-1)^2-4\gamma}}+2\gamma \sqrt{(\zeta-\gamma-1)^2-4\gamma}}{4\gamma^2\zeta^2}=\frac{1}{(1-\gamma)^3} .
        \label{eq.lhopital2}
    \end{align}
\end{proof}

\section{Proofs in Section \ref{sec:RegularizationImpactOnMisclassification}}
\subsection{Proof of Theorem \ref{theorem2regularized}}
\label{appendixtheorem2regularized}

\begin{proof}[Proof of Theorem \ref{theorem2regularized}]
The proof uses the same technique as in the Theorem \ref{theorem1}, the misclassification error is the same as \eqref{apeq:risk}, and we only need to show the limits of $q_0$ and $q_1$. By the change of variables formula in \eqref{apeq:changeofvar} and  Lemma \ref{lem.MPLaw}, we deduce 
\begin{align}
&\hat{\beta}^\top (\hat{\alpha} - \mu_0) \nonumber\\& = - \frac{1}{2} (\mu_d + \frac{1}{\sqrt{n_0}}z_0 - \frac{1}{\sqrt{n_1}}z_1)^\top \left(\frac{Z^\top Z}{n-2} + \lambda I_p\right)^\dagger (\mu_d - \frac{1}{\sqrt{n_0}} z_0 - \frac{1}{\sqrt{n_1}} z_1)\nonumber \\
& =\frac{1}{2n_0} z_0^\top \left(\frac{Z^\top Z}{n-2}+ \lambda I_p\right)^\dagger z_0  - \frac{1}{2}(\mu_d - z_1/\sqrt{n_1})^\top \left(\frac{Z^\top Z }{n-2}+ \lambda I_p\right)^\dagger (\mu_d - z_1/\sqrt{n_1})  \nonumber\\
& \overset{d}{=} \left[ -\frac{1}{2}\|\mu_d -\tfrac{1}{\sqrt{n_1}}z_1\|_2^2 + \frac{1}{2n_0} \|z_0\|_2^2 \right] \times \frac{1}{p} \tr\left(\frac{Z^\top Z}{n-2} + \lambda I_p\right)^{\dagger} \nonumber \\
&\arrowas -\frac{1}{2} (\Delta^2 + \gamma_0 -\gamma_1) m(-\lambda). 
\label{apeq:reg-numerator}
\end{align}
Similarly we can derive 
\begin{align}
\|\widehat\beta \|_\Sigma^2
= & (\widehat{\mu}_0-\widehat{\mu}_1)^\top (\widehat\Sigma+\lambda I_p)^\dagger \Sigma (\widehat\Sigma + \lambda I_p)^\dagger(\widehat{\mu}_0-\widehat{\mu}_1)\nonumber  \\
= & ({\mu}_d + \frac{1}{\sqrt{n_0}}{z}_0 - \frac{1}{\sqrt{n_1}}{z}_1)^\top  \left(\left(\frac{Z^\top Z}{n-2} + \lambda I_p\right)^{\dagger}\right)^2 ({\mu}_d + \frac{1}{\sqrt{n_0}}{z}_0 - \frac{1}{\sqrt{n_1}}{z}_1)\nonumber \\
\overset{d}{=}& \norm{\mu_d + \frac{1}{\sqrt{n_0}}z_0 - \frac{1}{\sqrt{n_1}}z_1}_2^2 \times \frac{1}{p} \tr\left(\left(\frac{Z^\top Z}{n-2}+\lambda I_p\right)^\dagger\right)^2\nonumber\\
\arrowas & (\Delta^2 + \gamma_0 + \gamma_1)m'(-\lambda). \label{apeq:reg-denominator}
\end{align}
Combining \eqref{apeq:reg-numerator} and \eqref{apeq:reg-denominator}, as well as putting the threshold term $\ln(n_1/n_0)$ back, we obtain
\begin{equation*}
    q_0 \arrowas \frac{-\frac{1}{2}(\Delta^2 -\gamma_0+\gamma_1)m(-\lambda)+\ln{\frac{\gamma_0}{\gamma_1}}}{\sqrt{(\Delta^2 + \gamma_1 + \gamma_0)m'(-\lambda)}}.
\end{equation*}

The same argument of analyzing $q_0$ applies to $q_1$ and therefore, we have 
\begin{equation*}
    q_1 \arrowas \frac{-\frac{1}{2}(\Delta^2 + \gamma_0 - \gamma_1)m(-\lambda)+\ln{\frac{\gamma_1}{\gamma_0}}}{\sqrt{(\Delta^2 + \gamma_1 + \gamma_0)m'(-\lambda)}}.
\end{equation*}
We complete the proof by substituting $q_0,q_1$ above into \eqref{apeq:risk}.
\end{proof}

\subsection{Proof of Regularized Phase Transition in Section \ref{sec:rldaphase}}\label{appendix:rldaphasetransition}
In this section, we show with a strong regularization, the phase transition phenomenon will vanish. Denote the asymptotic misclassification error in Theorem \ref{theorem2regularized}  as 
$$\cR_\lambda(\gamma_0, \gamma_1) = \sum_{\ell=0,1} \Phi\left(\frac{g(\gamma_0, \gamma_1, \ell)m(-\lambda)+(-1)^\ell\ln{\frac{\gamma_0}{\gamma_1}}}{k(\gamma_0, \gamma_1)\sqrt{m'(-\lambda)}}\right),$$
and we use the shorthand $\cR_\lambda(\gamma)$ to denote $\cR_\lambda(\gamma_0,\gamma_1)$ with the balanced data, i.e., $\gamma_0=\gamma_1 = 2\gamma$,
\begin{equation*}
    \cR_\lambda(\gamma) := \cR_\lambda(2\gamma,2\gamma) = \Phi\left(\frac{-\Delta^2 m(-\lambda)}{2\sqrt{(\Delta^2+4\gamma)m'(-\lambda)})}\right).
\end{equation*}
We show the phase transition phenomenon vanishes with a strong regularization, namely,
\begin{equation}
    \frac{\partial }{\partial \gamma_1} \cR_\lambda(\gamma_0,\gamma_1)\given_{\gamma_0=\gamma_1=2\gamma} > 0 \quad \text{for a strong $\lambda>0$}. \nonumber
\end{equation}
To see the result above, we need to show that $\frac{\partial \cR_\lambda(\gamma)}{\partial \gamma}>0$ with a strong regularization. Specifically, invoking Chain rule and by some manipulation, we have
$$\frac{\partial}{\partial \gamma_1}\cR_\lambda(\gamma_0, \gamma_1)\given_{\gamma_0=\gamma_1=2\gamma} = \frac{\partial \cR_\lambda(\gamma)  }{\partial \gamma}\frac{\partial \gamma}{\partial \gamma_1}\given_{\gamma_0=\gamma_1=2\gamma}.$$ 
By Mathematica Software \citep{Mathematica}, we check 
$$
    \frac{\partial \cR_\lambda(\gamma)}{\partial \gamma} =  \phi\left(\frac{-\Delta^2m(-\lambda)}{2\sqrt{(\Delta^2+4\gamma)m'(-\lambda)}}\right) \underbrace{\frac{\partial }{\partial \gamma} \left(\frac{-\Delta^2m(-\lambda)}{2\sqrt{(\Delta^2+4\gamma)m'(-\lambda)}}\right)}_{(\star)},
$$
where $\phi$ is the pdf of the standard normal distribution. By Mathematica Software \citep{Mathematica}, the denominator of $(\star)$ is also always positive, and is given by
\begin{align*}
&2 \sqrt{2} \gamma ^3 \lambda ^3 \frac{\left[(\gamma +\lambda +1)^2-4 \gamma \right]^{3/2}\left(4 \gamma +\Delta ^2\right)^{3/2}}{\left[\gamma  \lambda ^2 \left((\gamma +\lambda +1)^2-4 \gamma \right)\right]^{3/2}} 
\\& \times \Big[\gamma ^3 + \gamma ^2 (\sqrt{2 (\gamma +1) \lambda +(\gamma -1)^2+\lambda ^2}+2 \lambda -3)\\
&\quad+\gamma  \Big(-2 \sqrt{\gamma ^2+2 \gamma  (\lambda -1)+(\lambda +1)^2}+\lambda  \sqrt{2 (\gamma +1) \lambda +(\gamma -1)^2+\lambda ^2}+\lambda ^2+3\Big)\\
&\quad+(\lambda +1) (\sqrt{2 (\gamma +1) \lambda +(\gamma -1)^2+\lambda ^2}-\lambda -1)\Big]^{3/2}\\
&\hspace{2cm}.
\end{align*}
As a result, the sign of $\frac{\partial \cR_\lambda(\gamma)}{\partial \gamma}$ is determined by the numerator of $(\star)$ given as

\begin{align*}
  &\Delta ^2 \Big\{\gamma  (\lambda +1) \Big[\Delta ^2 \Big(-3 \sqrt{2 (\gamma +1) \lambda +(\gamma -1)^2+\lambda ^2}
  -2 \lambda \sqrt{2 (\gamma +1) \lambda +(\gamma -1)^2
  +\lambda ^2}\\
  &+\lambda ^2+5 \lambda+4\Big)+8 (\lambda +1)^2 \left(\sqrt{2 (\gamma +1) \lambda +(\gamma -1)^2+\lambda ^2}-\lambda -1\right)\Big]\\
  & +\gamma ^2 \Big[\Delta ^2 \Big(2 \lambda  \sqrt{2 (\gamma +1) \lambda +(\gamma -1)^2+\lambda ^2}+3 \sqrt{2 (\gamma +1) \lambda +(\gamma -1)^2+\lambda ^2}\\& -6 \lambda -6\Big) -4 (\lambda +1) \Big(2 \lambda  \sqrt{2 (\gamma +1) \lambda +(\gamma -1)^2+\lambda ^2}-9 \lambda -9\\&+7 \sqrt{2 (\gamma +1) \lambda +(\gamma -1)^2+\lambda ^2}\Big)\Big]
  +\gamma ^3 \Big[4 \lambda  \Big(\sqrt{\gamma ^2+2 \gamma  (\lambda -1)+(\lambda +1)^2}\\&+\lambda \sqrt{2 (\gamma +1) \lambda +(\gamma -1)^2+\lambda ^2}+\lambda ^2-9\Big)+36 \sqrt{\gamma ^2+2 \gamma  (\lambda -1)+(\lambda +1)^2}
  \\&+\Delta ^2 \left(-\sqrt{2 (\gamma +1) \lambda +(\gamma -1)^2+\lambda ^2}+\lambda +4\right)-64\Big] +\gamma ^4 \Big[-\Delta ^2+56 \\&-20 \sqrt{\gamma ^2+2 \gamma  (\lambda -1)+(\lambda +1)^2}+4 \lambda  \left(2 \sqrt{\gamma ^2+2 \gamma  (\lambda -1)+(\lambda +1)^2}+3 \lambda -4\right)\Big]
  \\&+4 \gamma ^5 \left[\sqrt{2 (\gamma +1) \lambda +(\gamma -1)^2+\lambda ^2}+3 \lambda -6\right]+4 \gamma ^6
  \\
  &+\Delta ^2 (\lambda +1)^3 \left(\sqrt{2 (\gamma +1) \lambda +(\gamma -1)^2+\lambda ^2}-\lambda -1\right)\Big\}.
\end{align*}

Combining the denominator and numerator, $\frac{\partial}{\partial \gamma} \cR_\lambda(\gamma)$ is positive only when one of the following case happens,
\begin{enumerate}
    \item $\Delta >0 ~{\rm and}~ 0<\gamma \leq 1~{\rm and}~\lambda >0$.
    \item $\Delta >0 ~{\rm and}~ \gamma \geq \frac{\sqrt{\Delta ^2+4}}{2}+1~{\rm and}~ \lambda >0$.
    \item $\Delta >0 ~{\rm and}~ 1<\gamma <\frac{\sqrt{\Delta ^2+4}}{2}+1 ~{\rm and}~ \lambda \geq \text{the smallest real root of}\\ 
    \big[\text{$\#$1}^4 (32 \gamma +4 \Delta ^2)+\text{$\#$1}^3 (96 \gamma ^2+8 \gamma  \Delta ^2+128 \gamma +16 \Delta ^2)
    +\text{$\#$1}^2 (96 \gamma ^3+112 \gamma ^2+16 \gamma  \Delta ^2+192 \gamma +\Delta ^4+24 \Delta ^2)
    +\text{$\#$1} (-8 \gamma ^3 \Delta ^2-16 \gamma ^2 \Delta ^2+32 \gamma ^4-96 \gamma ^3-64 \gamma ^2-2 \gamma  \Delta ^4+8 \gamma  \Delta ^2+128 \gamma+2 \Delta ^4+16 \Delta ^2) -4 \gamma ^4 \Delta ^2 +16 \gamma ^3 \Delta ^2+\gamma ^2 \Delta ^4-16 \gamma ^2 \Delta ^2-16 \gamma ^4+64 \gamma ^3- 80 \gamma ^2-2 \gamma  \Delta ^4+32 \gamma +\Delta ^4+4 \Delta ^2\big]$.
\end{enumerate}

Consequently, We deduce that the misclassification error increases when $\gamma$ grows in the interval $\left(0, 1\right)$ or $\left(\frac{\sqrt{\Delta^2+4}}{2}, \infty\right)$; when $\gamma$ grows in $\left(1,\frac{\sqrt{\Delta^2+4}}{2}\right)$, the misclassification error decreases when $\lambda$ is small, yet increases when $\lambda$ is large. For example, when $\Delta^2=9$ and $\lambda=1$, $\cR_\lambda(\gamma)$ increases monotonically with respect to $\gamma$, and the peaking phenomenon disappears. Meanwhile we have the instantaneous derivative $\frac{\partial \cR_\lambda(\gamma_0,\gamma_1)}{\partial \gamma_1}\given_{\gamma_0=\gamma_1=2\gamma} >0$ for any $\gamma_0$, which implies that the phase transition phenomenon vanishes.

The commands of Mathematica are provided as follows.

\hspace{-0.5cm}In[1]:
$m(\gamma \_,\lambda \_)\text{:=}\frac{\sqrt{(\gamma +\lambda +1)^2-4 \gamma }+\gamma -\lambda -1}{2 \gamma  \lambda }$

\hspace{-0.5cm}In[2]:
$R(\gamma \_,\lambda \_,\Delta \_)\text{:=}-\frac{\Delta ^2 m(\gamma ,\lambda )}{2 \sqrt{-\left(4 \gamma +\Delta ^2\right) \frac{\partial m(\gamma ,\lambda )}{\partial \lambda }}}$

\hspace{-0.5cm}In[3]:
$\text{de}(\gamma \_,\lambda \_,\Delta \_)\text{:=}$

\hspace{1.5cm}$\text{Evaluate}\left[\text{Denominator}\left[\text{FullSimplify}\left[\text{Together}\left[\frac{\partial R(\gamma ,\lambda ,\Delta )}{\partial \gamma }\right]\right]\right]\right]$

\hspace{-0.5cm}In[4]: 
$\text{nu}(\gamma\_ ,\lambda\_ ,\Delta\_ )\text{:=}$

\hspace{1.5cm}$\text{Evaluate}\left[\text{Numerator}\left[\text{FullSimplify}\left[\text{Together}\left[\frac{\partial R(\gamma ,\lambda ,\Delta )}{\partial \gamma }\right]\right]\right]\right]$

\hspace{-0.5cm}In[5]:
$\text{Reduce}[\text{de}(\gamma ,\lambda ,\Delta )\geq 0\land \lambda >0\land \gamma >0\land \Delta >0,\{\gamma ,\lambda \}]$

\hspace{-0.5cm}In[6]: 
$\text{Reduce}[\text{nu}(\gamma ,\lambda ,\Delta )\geq 0\land \lambda >0\land \gamma >0\land \Delta >0,\{\gamma ,\lambda \}]$

\hspace{-0.5cm}In[7]:
$\text{Reduce}[\text{nu}(\gamma ,1,3)\geq 0\land \gamma >0,\{\gamma \}]$

\bibliography{ref}    
\bibliographystyle{mplainnat}
\end{document}